%% file: main.tex
\documentclass[a4paper,onecolumn, superscriptaddress,10pt,
accepted=2021-02-02, issue=1, volume=3,
shorttitle=papers/compositionality-3-1]{compositionalityarticle}
\pdfoutput=1
\usepackage[english]{babel}
\usepackage{amsmath}
\usepackage{hyperref}

\usepackage{tikz}

\usepackage{amsthm}
\usepackage{amsmath}
\usepackage{amssymb}
\usepackage{mathtools}
\usepackage[square]{natbib}
\setcitestyle{numbers}
\usepackage{listings}
\usepackage{float} 

\usetikzlibrary{
  cd,
  math,
  decorations.markings,
  decorations.pathreplacing,
  positioning,
  arrows.meta,
  circuits.logic.US,
  shapes,
  calc,
  fit,
  trees,
  quotes}
\usetikzlibrary{decorations.pathmorphing}
\usetikzlibrary{decorations.markings}
\usetikzlibrary{decorations.pathreplacing}
\usetikzlibrary{arrows}
\usetikzlibrary{shapes.geometric}

\pgfdeclarelayer{edgelayer}
\pgfdeclarelayer{nodelayer}
\pgfsetlayers{edgelayer,nodelayer,main}
\tikzstyle{none}=[inner sep=0pt]
\usetikzlibrary{shapes.geometric}%   

\input{header}

\newtheorem{proposition}{Proposition}

\newtheorem{definition}{Definition}[section]

\binoppenalty=\maxdimen
\relpenalty=\maxdimen
\begin{document}

\title{Categorical Stochastic Processes and Likelihood}
\date{}
\author{Dan Shiebler}
\email{daniel.shiebler@kellogg.ox.ac.uk}
\homepage{danshiebler.com}
\affiliation{Department for Continuing Education and Department of Computer Science\\
University of Oxford, Oxford, United Kingdom
}
\maketitle

% \author{Brendan Fong}
% \email{latex@compositionality-journal.org}
% \homepage{http://compositionality-journal.org}
% \orcid{0000-0003-0290-4698}
% \thanks{You can use the \texttt{\textbackslash{}email}, \texttt{\textbackslash{}homepage}, and \texttt{\textbackslash{}thanks} commands to add additional information for the preceding \texttt{\textbackslash{}author}. If applicable, this can also be used to indicate that a work has previously been published in conference proceedings.}
% \affiliation{MIT, Mathematics Department, Boston, USA}

\begin{abstract}
We take a category-theoretic perspective on the relationship between probabilistic modeling and gradient based optimization. We define two extensions of function composition to stochastic process subordination: one based on a co-Kleisli category and one based on the parameterization of a category with a Lawvere theory. We show how these extensions relate to the category of Markov kernels $\stoch$ through a pushforward procedure.
%   \cite{lawvereprob} \cite{giry1982categorical} and other Markov Categories \cite{fritz2019synthetic}.
  
We extend stochastic processes to parametric statistical models and define a way to compose the likelihood functions of these models. We demonstrate how the maximum likelihood estimation procedure defines a family of identity-on-objects functors from categories of statistical models to the category of supervised learning algorithms $\learn$. 
%   in \cite{fong2017backprop}.
  
Code to accompany this paper can be found on GitHub\footnote{ \url{https://github.com/dshieble/Categorical_Stochastic_Processes_and_Likelihood}}.
\end{abstract}

\input{body}

\bibliography{main.bib}
\bibliographystyle{plainnat}
\clearpage
% \section{Appendix A: Optimization}\label{loss}\input{6-loss}

% \bibliographystyle{plain}

% \begin{thebibliography}{9}
% \bibitem{examplecitation}
%   Name Surname,
%   \href{https://doi.org/10.22331/
%         idonotexist}{Compositionality
%         \textbf{123}, 123456 (1916).}

% \bibitem{biblatexsubmittingtothearxiv}
%   StackExchange discussion on \href{http://tex.stackexchange.com/questions/26990/biblatex-submitting-to-the-arxiv}{``Biblatex: submitting to the arXiv'' (2017-01-10)}

% \bibitem{arxivpdfoutput}
%   Help article published by the arXiv on \href{https://arxiv.org/help/submit_tex}{``Considerations for TeX Submissions'' (2017-01-10)}

% \bibitem{howtogetdoilinksinbibliography}
%   StackExchange discussion on \href{http://tex.stackexchange.com/questions/3802/how-to-get-doi-links-in-bibliography}{``How to get DOI links in bibliography'' (2016-11-18)}
  
% \bibitem{automaticallyaddingdoifieldstoahandmadebibliography}
%   StackExchange discussion on \href{http://tex.stackexchange.com/questions/6810/automatically-adding-doi-fields-to-a-hand-made-bibliography}{``Automatically adding DOI fields to a hand-made bibliography'' (2016-11-18)}
% \end{thebibliography}

% \onecolumn\newpage

\end{document}

%% file: header.tex
\newcommand{\cb}{\mathbf{C}}
\newcommand{\db}{\mathbf{D}}

\DeclareMathOperator{\del}{\mathsf{del}}
\DeclareMathOperator{\cp}{\mathsf{cp}}

\newcommand{\baseprobn}{(\Omega^n, \bc(\Omega^n), \mu^n)}
\newcommand{\baseprob}{(\Omega, \bo, \mu)}
\newcommand{\base}{(\Omega, \bo, \mu)}
\newcommand{\bc}{\mathcal{B}}

\newcommand{\bo}{\bc(\Omega)}
\newcommand{\bmeas}{\mathbf{BorelMeas}}
\newcommand{\bstoch}{\mathbf{BorelStoch}}

\newcommand{\cbr}{\cb_{\mathcal{R}_{\mu}}}
\newcommand{\ceucmeas}{\mathbf{C}\eucmeas}

\newcommand{\cl}{\mathbf{CondLikelihood}}

\newcommand{\coklo}{\mathbf{CoKl}_{O}}

\newcommand{\copyfunctor}{Copy}

\newcommand{\dfn}{\df_{\mathcal{N}_{\mu}}}
\newcommand{\dfr}{\df_{\mathcal{R}_{\mu}}}
\newcommand{\df}{\mathbf{DF}}

\newcommand{\euc}{\mathbf{Euc}}
\newcommand{\eucmeas}{\mathbf{EucMeas}}

\newcommand{\expectation}{Exp}

\newcommand{\finstoch}{\mathbf{FinStoch}}

\newcommand{\gauss}{\mathbf{Gauss}}

\newcommand{\learn}{\mathbf{Learn}}

\newcommand{\meas}{\mathbf{Meas}}

\newcommand{\nm}{\mathcal{N}_{\mu}}

\newcommand{\obo}{(\Omega, \bc(\Omega))^{*}}

\newcommand{\parac}{\para(\cb)}
\newcommand{\parad}{\para_{\mathbf{D}}}

\newcommand{\parao}{\para_{(\Omega, \bo)^{*}}}

\newcommand{\para}{\mathbf{Para}}
\newcommand{\peuc}{\mathbf{PEuc}}

\newcommand{\push}{Push_{\mu}}

\newcommand{\realize}{R_{\omega}}

\newcommand{\rl}{\mathbb{R}}

\newcommand{\rn}{\mathcal{RN}_{\mu}}

\newcommand{\stoch}{\mathbf{Stoch}}

%% file: body.tex
\section{Introduction}
Many machine learning algorithms contain an irreducible aspect of randomness. Using category theory to reason about how this randomness breaks down into compositional and functorial structure helps us build a high-level picture of probabilistic learning and its connections to other fields. There are two kinds of uncertainty that most probabilistic reasoning aims to capture:
\begin{definition}
Epistemic uncertainty is uncertainty due to limited data or knowledge.
\end{definition}
For example, if we have a very small amount of data then we need to cope with high epistemic uncertainty. Cho et al. \cite{cho2019disintegration} and Culbertson et al. \cite{culbertson2014categorical, culbertson2013bayesian} explore how new data points affect their models' epistemic uncertainty. For example, a simple model of a complex nonlinear system is likely to have high epistemic uncertainty. 
\begin{definition}
Aleatoric uncertainty is inherent uncertainty in a system that can cause results to differ each time we run the same experiment
\end{definition}
For example, if we aim to predict the output of a system that includes a non-deterministic stage (such as a coin toss), we will need to cope with aleatoric uncertainty. Aleatoric uncertainty is common in physical systems. For example, many biological processes will produce slightly different results based on randomness in turbulent fluid flows.
% epistemic - based on environmental unknowns such as the concentration of substrates in the blood, air temperature, etc \cite{indrayan2017medical}
For this reason, models that approximate physical systems often implicitly or explicitly produce a probability distribution over the possible outputs conditioned on some input \citep{Walker365973}.

Even models that produce point estimates, such as the ones described by Fong et al. \cite{fong2019backprop}, can be viewed as predicting the expected value of some unknown probability distribution. For example, suppose we have some system $X \rightarrow y$ that contains a degree of aleatoric uncertainty such that $P(y | X)$ is Gaussian. Now suppose we train a point estimate model that predicts $y$ from $X$ such that the mean square error between the model's predictions and the observations from the execution of this system is minimized. This is approximately equivalent to minimizing the Kullback-Leibler (KL) divergence (which measures how one probability distribution is different from a second, reference probability distribution) between a distribution with expected value given by the model's output and $P(y | X)$. In this way the structure of the model's aleatoric uncertainty is captured in its loss function (mean square error in this case). 

Now consider a physical system which has several components, each of which has some degree of aleatoric uncertainty. Suppose we want to build a compositional model for this system. If we use the neural network-like composition of Fong et al.'s \cite{fong2019backprop}, then we can only represent the full model's uncertainty with the loss function that parameterizes the backpropagation functor. As a result, we cannot characterize the interactions between the uncertainty in the different parts of the system.

For example, Eberhardt et al. \cite{eberhardt2016deep} build a convolutional neural network model to assess how the visual cortex performs a rapid stimulus categorization task. Their model includes multiple layers which represent the hierarchy within the central nervous system from photorecepters in the eye, to edge-detecting neurons in the primary visual cortex, to higher-order feature detectors in the later stages of visual cortex. Although there is aleatoric uncertainty at each layer of this biological system, Eberhardt et al. use a standard composition of neural network layers and therefore can only represent this uncertainty with a cross-entropy loss over the model's final output.

We describe an alternative strategy for constructing and composing parametric models such that we can explicitly characterize how different subsystems' uncertainties interact. We use this strategy to build a generalized framework for training neural networks that have stochastic processes as layers. To do this, we replace the domain of Fong et al.'s \cite{fong2019backprop} Backpropagation functor $\para$ with a probabilistically motivated category over which we can define the error function $er: \rl \times \rl \rightarrow \rl$ through the maximum likelihood procedure. Our specific contributions are to:
\begin{itemize}
    \item Develop a strategy for composing stochastic processes that is compatible with both subordination \citep{lalley2007levy} and parametric function composition \citep{fong2019backprop}.
    \item Introduce two categories with this compositional structure, one based on $\para$ \citep{fong2019backprop} and one based on the co-Kleisli category of the product comonad, and explore their relationships with each other and with the category $\stoch$ of Markov kernels.
    \item Extend the category of stochastic processes to a category of parametric statistical models.
    \item Demonstrate that the Radon-Nikodym derivative with respect to the Lebesgue measure acts as a semifunctor from a sub-semicategory of parametric statistical models into a semicategory of likelihood functions.
    \item Define a family of subcategories of parametric statistical models over which we can use the maximum likelihood procedure to define a backpropagation functor into the category $\learn$ of learning algorithms \citep{fong2019backprop}.
\end{itemize}

\section{Preliminaries}

%%%%%%%%%%%%%%%%%%%%%%%%%%%%%%%%%%%%%%%%%%

\subsection{Probability}
% Largely pulled from Section 2 in https://arxiv.org/pdf/1301.6201.pdf

\subsubsection{Probability Measures, Random Variables and Markov Kernels}\label{markovkernels}

\begin{definition}
%(Definition from \cite{fong2013causal})
A $\sigma$-algebra $\Sigma$ on the set $\Omega$ is a set of subsets of $\Omega$ that includes $\Omega$ and is closed under (1) countable union (2) countable intersection (3) complement in $\Omega$.
%and contains the empty set.
\end{definition}

\begin{definition}\label{definition:measurable-space}
A measurable space is a pair $(\Omega, \Sigma)$ of a set $\Omega$ and a $\sigma$-algebra $\Sigma$ on $\Omega$.
\end{definition}
There is a natural notion of a morphism that we can define between measurable spaces.
\begin{definition}\label{definition:measurable-function}
A function $f: A \rightarrow B$ between the measurable spaces $(A, \Sigma_A)$ and $(B, \Sigma_B)$ is measurable if for any $\sigma_B \in \Sigma_B$, $f^{-1}(\sigma_B) \in \Sigma_A$.
\end{definition}
Given two measurable spaces we can take their product in a canonical way.
\begin{definition}\label{definition:product-measurable-space}
Given measurable spaces $(\Omega, \Sigma_{\Omega})$, $(\Omega', \Sigma_{\Omega'} )$ we can form a $\sigma$-algebra $\Sigma_{\Omega} \times \Sigma_{\Omega'}$ on the set $\Omega \times \Omega'$ by taking all countable unions and complements of subsets in $\{\sigma_{\Omega} \times \sigma_{\Omega'} \ |\  \sigma_{\Omega} \in \Sigma_{\Omega} , \sigma_{\Omega'} \in \Sigma_{\Omega'}\}$. We call $\Sigma_{\Omega} \times \Sigma_{\Omega'}$ the product $\sigma$-algebra of $\Sigma_{\Omega}$ and $\Sigma_{\Omega'}$ and we call the measurable space $(\Omega \times \Omega', \Sigma_{\Omega} \times \Sigma_{\Omega'})$ the product measurable space of $(\Omega, \Sigma_{\Omega}), (\Omega', \Sigma_{\Omega'})$.
\end{definition}

We can also form measurable spaces from topological spaces.
\begin{definition}
% Example 2.3 in https://arxiv.org/pdf/1301.6201.pdf
Given the topological space $\Omega$ the Borel $\sigma$-algebra $\bc(\Omega)$ of $\Omega$ is the $\sigma$-algebra generated by the collection of open subsets of $\Omega$. We call the measurable space $(\Omega, \bc(\Omega))$ a Borel measurable space. 
\end{definition}
%
% NOTE: We cannot define Lebesgue measure on all subsets of \rl. There are some subsets that are too weirdly shaped (i.e. assemble a pea into the sun paradox). Instead we define "Lebesgue measurable" sets and define the Lebesgue measure over the Lebesgue algebra of R^n L(R^n). See Definition 2.8 of https://www.math.ucdavis.edu/~hunter/m206/ch1_measure.pdf.
%
% NOTE: If we define B(R^n) to be the Borel $\sigma$-algebra of R^n then by Proposition 2.21 in https://www.math.ucdavis.edu/~hunter/m206/ch1_measure.pdf we have B(R^n) \subset L(R^n) 
%
Said another way, the Borel $\sigma$-algebra of $\Omega$ is the smallest $\sigma$-algebra of $\Omega$ that contains all open sets in $\Omega$. The Borel measurable space associated with a Polish space (e.g. $\rl$) has special properties.
%
% NOTE: R^n is also a Polish space because (1) it is separable (https://en.wikipedia.org/wiki/Separable_space) and (2) it is complete with the Euclidean metric (https://en.wikipedia.org/wiki/Complete_metric_space)
%
\begin{definition}\label{definition:standard-borel}
% Closedness under countable products is in Properties here: https://en.wikipedia.org/wiki/Standard_Borel_space
A standard Borel space is a Borel measurable space associated with a Polish space. Standard Borel spaces are closed under countable products.
\end{definition}
%
% NOTE: The n-product of the Borel space (R, B(R))^n is the same as the Borel space (R^n, B(R^n)) since:
%   (1) the n-product of R is R^n
%   (2) Since R^n is separable (rational vectors) by theorem 3b in https://encyclopediaofmath.org/wiki/Standard_Borel_space we have that B(R^n) = B(R)^n
%
Examples of standard Borel spaces include $(\rl, \bc(\rl))$ and its countable products $(\rl^n, \bc(\rl^n))$. 

The fundamental objects in measure-theoretic probability are the probability measure and probability space:
\begin{definition}
A probability space is a triplet $(\Omega, \Sigma, \mu)$ where $(\Omega, \Sigma)$ is a measurable space (Definition \ref{definition:measurable-space}) and $\mu$ is a probability measure over $(\Omega, \Sigma)$. That is, $\mu$ is a countably additive function over the $\sigma$-algebra $\Sigma$ that returns results in the unit interval $[0,1]$ such that $\mu(\Omega) = 1, \mu(\varnothing) = 0$. Recall that $\Sigma$ is a set of subsets of $\Omega$. 
\end{definition}
%
% TODO: Potentially add a reference for this?
Probability spaces are closed under products. This is, when $(\Omega, \Sigma, \mu)$ and $(\Omega', \Sigma', \mu')$ are probability spaces the product space $(\Omega \times \Omega', \Sigma \times \Sigma', \mu\mu')$ where $\mu\mu'(\omega) = \mu(\omega)\mu'(\omega)$ is also a probability space.

In practice, we will generally work with parameterized probability measures, which we call Markov kernels.
\begin{definition}
A Markov kernel between the measurable space $(A, \Sigma_A)$ and the measurable space $(B, \Sigma_B)$ is a function $\mu:  A \times \Sigma_B \rightarrow [0,1]$ such that:
\begin{itemize}
    \item For all $\sigma_b \in \Sigma_B$, the function $\mu(\_, \sigma_b): A \rightarrow [0,1]$ is measurable.
    \item For all $x_a \in A$, $\mu(x_a, \_): \Sigma_B \rightarrow [0,1]$ is a probability measure on $(B, \Sigma_B)$. In particular:
    \begin{gather*}
        \mu(x_a, B) = 1
        \qquad
        \mu(x_a, \varnothing) = 0.
    \end{gather*}
\end{itemize}
\end{definition}
For example, a Markov Kernel between the one-point set and the measurable space $(A, \Sigma_A)$ is just a probability measure over $(A, \Sigma_A)$.

Another foundational object in measure-theoretic probability is the random variable, which is paradoxically neither random nor a variable.
\begin{definition}
A random variable defined on the probability space $(\Omega, \Sigma, \mu)$ is a measurable function from $(\Omega, \Sigma)$ to $(\rl, \bc(\rl))$. 
\end{definition}
% In this paper we will assume $\Sigma_{\rl}$ is $\bc(\rl)$, or the Borel $\sigma$-algebra of $\rl$.
We will sometimes use the term ``random variable'' to refer to measurable functions into $(\rl^n, \bc(\rl^n))$ as well. These are also called multivariate random variables or random vectors. While some authors use uppercase letters like $X$ to denote random variables, we will use lowercase letters like $f,g$ to emphasize that random variables are functions. 

Random variables and probability measures are closely related:
\begin{definition}
Given a probability space $\baseprob$ and a random variable $f: \Omega \rightarrow \rl$, the pushforward $f_{*}\mu$ of $\mu$ along $f$ is a probability measure over $(\rl, \bc(\rl))$ defined to be:
\begin{gather*}
    f_{*}\mu: \bc(\rl) \rightarrow [0,1]
    \\
    f_{*}\mu(\sigma_{\rl}) = \mu(f^{-1}(\sigma_{\rl})).
\end{gather*}
\end{definition}

Like probability measures, random variables have a parameterized extension.
\begin{definition}
A stochastic process defined in the probability space $(\Omega, \Sigma, \mu)$ is a family of random variables indexed by some set $T$.
\end{definition}
That is, we can write a stochastic process as a function $f: \Omega \times T \rightarrow \rl$. We limit our study to stochastic processes that are jointly Borel-measurable. We can define the pushforward of $\mu$ along such a stochastic process $f$ to be the following Markov Kernel:
\begin{gather*}
    f_{*}\mu: T \times \bc(\rl) \rightarrow [0,1]
    \\
    f_{*}\mu(x_t, \sigma_{\rl}) =
    f(\_, x_t)_{*}\mu(\sigma_{\rl}) =
    \mu(f(\_, x_t)^{-1}(\sigma_{\rl})).
\end{gather*}
% Inverse of borel-measurable bijection is also borel measurable https://math.stackexchange.com/questions/56022/measurability-of-the-inverse-of-a-measurable-function

\subsubsection{Categories in Probability}

Measurable spaces and measurable functions form a symmetric monoidal category as follows:
\begin{definition}\label{definition:meas}
% NOTE: I am pretty sure that \meas is not a strict symmetric monoidal category
The objects in $\meas$ are pairs $(A, \Sigma_A)$, where $\Sigma_A$ is a $\sigma$-algebra over $A$, and morphisms are measurable functions.
% By Page 24 of https://arxiv.org/pdf/1301.6201.pdf the measurable space product is the categorical product
% 
% Measurable space products: https://damekdavis.wordpress.com/2010/12/13/is-the-product-of-measurable-spaces-the-categorical-product/
The tensor product of the measurable spaces $(A, \Sigma_A)$ and $(B, \Sigma_B)$ in $\meas$ is the product measurable space $(A \times B, \Sigma_A \times \Sigma_B)$ (Definition \ref{definition:product-measurable-space}) and the tensor product of the measurable functions $g,f$ is the function $(g\otimes f)(x,y) = (g(x),f(y))$.
% , where $\Sigma_A \times \Sigma_B$ is the product $\sigma$-algebra of $\Sigma_A$ and $\Sigma_B$.
\end{definition}

A notable subcategory of $\meas$ is the following:
\begin{definition}\label{definition:bmeas}
The category $\bmeas$ is the symmetric monoidal subcategory of $\meas$ in which objects are limited to standard Borel measurable spaces.
\end{definition}
Since standard Borel spaces are closed under countable products this subcategory is symmetric monoidal.

\begin{proposition}\label{proposition:eucmeas}
% NOTE: Infinite differentiability is the extreme of smoothness https://en.wikipedia.org/wiki/Smoothness
%
% NOTE: continuous differentiability implies absolute continuity
%
% NOTE: I think that EucMeas is not actually a subcategory of Meas because the tensor is defined slightly differently. This is because tensor in Meas is not actually associative on the nose since (A x B) x C is a different set from A x (B x C). EucMeas is probably isomorphic to a subcategory of Meas instead???
We can form a strict symmetric monoidal subcategory $\eucmeas$ of $\bmeas$ in which objects are restricted to be $(\rl^n, \bc(\rl^n))$ for some $n \in \mathbb{N}$, the tensor of the objects $(\rl^n, \bc(\rl^n))$ and $(\rl^m, \bc(\rl^m))$ is $(\rl^{n+m}, \bc(\rl^{n+m}))$ and morphisms are restricted to be infinitely differentiable.
\end{proposition}
\begin{proof}
We first need to show that $\eucmeas$ is a symmetric monoidal subcategory of $\meas$. $\eucmeas$ contains all identities since the identity function $f(x_n) = x_n$ is infinitely differentiable for all $n$.
% Infinite differentiability is closed under composition and tensor https://kaba.hilvi.org/homepage/blog/differentiable.htm
Next, given two infinitely differentiable measurable maps $g,f$ the composition $g \circ f$ and tensor $(g \otimes f)(x,y) = (g(x),f(y))$ are infinitely differentiable and measurable as well. Therefore morphisms in $\eucmeas$ are closed under composition and tensor.
Next, since the tensor product of the standard Borel spaces $(\rl^n, \bc(\rl^n))$ and $(\rl^m, \bc(\rl^m))$ in $\meas$ is the standard Borel space $(\rl^{n+m}, \bc(\rl^{n+m}))$ we have that objects in $\eucmeas$ are closed under tensor as well.  

Next, in order to show that $\eucmeas$ is strict monoidal we need to show that the associators and unitors in $\eucmeas$ are identities.  First, since:
\begin{gather*}
    (\rl^{n+(m+k)}, \bc(\rl^{n+(m+k)})) = (\rl^{(n+m)+k}, \bc(\rl^{(n+m)+k}))
\end{gather*}
the associators in $\eucmeas$ are identities. Next, since:
\begin{gather*}
    (\rl^{0+n}, \bc(\rl^{0+n})) = 
    (\rl^{n+0}, \bc(\rl^{n+0})) = 
    (\rl^{n}, \bc(\rl^{n}))
\end{gather*}
the unitors in $\eucmeas$ are identities.

\end{proof}

Another important category that we will consider is $\stoch$ \citep{lawvereprob, giry1982categorical}.
\begin{definition}\label{definition:stoch}
% NOTE: Stoch is not strict monoidal
%
In the symmetric monoidal category $\stoch$ objects are measurable spaces and morphisms are Markov kernels. We define the composition of the Markov kernels $\mu:  A \times \Sigma_B \rightarrow [0,1]$ and $\mu': B \times \Sigma_C \rightarrow [0,1]$ to be the following, where $x_a \in A$ and $\sigma_c \in \Sigma_C$:
\begin{gather*}
    (\mu' \circ \mu): A \times \Sigma_C \rightarrow [0,1]
    \\
    (\mu' \circ \mu)(x_a, \sigma_c) = 
    \int_{x_b \in B} \mu'(x_b,\sigma_c) d\mu(x_a, \_).
\end{gather*}
The identity morphism at $(A, \Sigma_A)$ is $\delta$ where for $x_a \in A, \sigma_a \in \Sigma_A$:
\begin{gather*}
    \delta: A \times \Sigma_A \rightarrow [0,1]
    \\
    \delta(x_a, \sigma_a) = \begin{cases} 
      1 & {x_a} \in \sigma_a \\
      0 & {x_a} \not\in \sigma_a \\
    \end{cases}.
    % (\delta \circ \mu)_{x}(\sigma_y) = \int_{y' \in  x_y} \delta_{y'}(\sigma_{y}) d\mu_x  = \int_{y' \in \sigma_{y}} d\mu_x = \mu_x(\sigma_{y}) \\
    % (\mu \circ \delta)_{x}(\sigma_y) = \int_{y' \in  x_y} \mu_{y'}(\sigma_y) d\delta_x  = \int_{y' \in  x_y} \mu_{y'}(\sigma_y) \delta_x(d\sigma_y') = \int_{y' \in \{x\}} \mu_{y'}(\sigma_y) d\sigma_y' = \mu_{x}(\sigma_y)
\end{gather*}
The tensor product of the Markov Kernels $\mu: A \times \Sigma_B \rightarrow [0,1]$ and $\mu':  C \times \Sigma_D \rightarrow [0,1]$ in $\stoch$ is the Markov Kernel:
\begin{gather*}
    (\mu' \otimes \mu): (A \times C) \times (\Sigma_B \times \Sigma_D) \rightarrow [0,1]
\end{gather*}
where $\Sigma_B \times \Sigma_D$ is the product sigma-algebra and for $x_a \in A, x_c \in C, \sigma_{b} \in \Sigma_B, \sigma_{d} \in \Sigma_D$:
\begin{gather*}
    (\mu' \otimes \mu)((x_a, x_c), \sigma_{b}\times \sigma_{d}) = \mu(x_a, \sigma_b)\mu(x_c, \sigma_d).
\end{gather*}
\end{definition}

% $\stoch$ naturally arises as the Kleisli category of the Giry Monad, which is an affine symmetric monoidal monad that sends a measurable space to the space of probability measures over that space  \citep{giry1982categorical}. 

The objects in $\stoch$ are also equipped with a commutative comonoidal structure that is compatible with the monoidal product in $\stoch$. That is, each metric space $X \in \stoch$ is equipped with a comultiplication map $\cp: X \rightarrow X \otimes X$ and a counit map $\del: X \rightarrow 1$ that satisfy the commutative comonoid equations, naturality of $\del$ and:
\begin{gather*}
    \cp_{X \otimes Y} = (id_X \otimes \sigma_{Y,X} \otimes id_Y)(\cp_{X}\otimes \cp_{Y}),
\end{gather*}
where $\sigma_{Y,X}: X \times Y \rightarrow Y \times X$ is the symmetric monoidal swap map in $\stoch$.

$\stoch$ has many notable subcategories. For example, if we limit to countably generated measurable spaces
as objects and Markov kernels over perfect probability measures as morphisms we get the following category:
\begin{definition}\label{definition:cgstoch}
The category $\mathbf{CGStoch}$ is the subcategory of $\stoch$ in which objects are limited to countably generated measurable spaces.
\end{definition}
If we add an additional condition of separability we form the following:
\begin{definition}\label{definition:bstoch}
% NOTE: If we ever add back in the monoidal stuff to the probability work we will need to define a EucStoch since the tensor product is now defined differently between EucMeas and BorelStoch
The category $\bstoch$ is the subcategory of $\stoch$ in which objects are limited to standard Borel spaces (the Borel spaces associated with Polish spaces).
\end{definition}
Limiting further to finite spaces yields the following category:
\begin{definition}\label{definition:finstoch}
The category $\finstoch$ is the subcategory of $\stoch$ in which objects are limited to finite measurable spaces.
\end{definition}

%%%%%%%%%%%%%%%%%%%%%%%%%%%%%%%%%%%%%%%%%%%%%
\subsection{Parameterized Morphisms}

A categorical tool that has risen in prominence to represent parameters in machine learning models is the $\para$ operator. This operator has several presentations \citep{capucci2021foundations, fong2019backprop, gavranovic2019compositional, cruttwell2021categorical}. A simple definition is as follows:

\begin{definition}\label{definition:para}
% NOTE: Para is only a symmetric monoidal category when the underlying category is itself commutative monoidal https://ncatlab.org/nlab/show/commutative+monoidal+category
%
% NOTE: Definition 2.1 in https://arxiv.org/pdf/2103.01931.pdf
Let $\cb$ be a strict symmetric monoidal category. Then $\parac$ is a category with the same objects as $\cb$. A morphism $A \rightarrow B$ in $\parac$ is a pair $(P, f)$ where $P$ is an object of $\cb$ and:
\begin{gather*}
f : P \otimes A \rightarrow B
\end{gather*}
is a morphism in $\cb$.
% %
% A $2$-cell from $(P, f)$ to $(P', f')$ is a morphism in $r : P' \rightarrow P$ in $\cb$ such that the following diagram commutes in $\cb$:
% %
% \begin{equation*}
% %
% \begin{tikzcd}[column sep=20pt, row sep=8pt]
% P' \otimes A \arrow[rd, "f'"'] \arrow[rr, "r \otimes id_{A}"] && P \otimes A \arrow[dl, "f"]\\
%   & B & \\
% \end{tikzcd}
% %
% \label{eq:reparam_triangle}
% \end{equation*}
%
The composition of morphisms:
\begin{align*}
    f: P \otimes A \rightarrow B
    \qquad
    g: Q \otimes B \rightarrow C
\end{align*}
is given by $(Q \otimes P, g \circ (id_{Q} \otimes f))$:
\begin{align*}
    & (Q \otimes P) \otimes A
    \xrightarrow{=} Q \otimes (P \otimes A) 
    \xrightarrow{id_{Q} \otimes f} Q \otimes B
    \xrightarrow{g} C
\end{align*}
\end{definition}

%%%%%%%%%%%%%%%%%%%%%%%%%%%%%%%%%%%%%%%%%%%%%%%%
\section{Random Variables and Independence}
\subsection{Random Variables and Independence in $\bstoch$}

In any categorical presentation of probability, a natural question is how to reason about the notion of independence of random variables \citep{gerhold2016categorial, franz2002stochastic, fritz2020synthetic}.

Since $\bstoch$ (Definition \ref{definition:bstoch}) is the Kleisli category of the restriction of the Giry monad  \citep{giry1982categorical} over $\bmeas$ (Definition \ref{definition:bmeas}), we can define an embedding functor from $\bmeas$ into $\bstoch$ that acts as an identity on objects and sends the measurable function $f: (A, \Sigma_A) \rightarrow (B, \Sigma_B)$ to the Dirac Markov kernel $\delta_{f}: A \times \Sigma_B \rightarrow [0,1]$ where for $x_a \in A, \sigma_b \in  \Sigma_B$:
\begin{gather*}
    \delta_{f}(x_a, \sigma_b) = \begin{cases} 
      1 & f(x_a)\in \sigma_b \\
      0 & f(x_a)\not\in \sigma_b \\
    \end{cases}
\end{gather*}
This formalizes the intuition that Markov kernels are a generalization of both measurable functions and probability measures, and provides an avenue to directly study random variables and their independence in $\bstoch$. 

Now suppose we have a probability space $(\Omega, \Sigma, \mu)$ such that $(\Omega, \Sigma)$ is standard Borel, and two real-valued random variables defined on this space $f, f'$. We can think of these random variables as morphisms in $\bmeas$ from $(\Omega, \Sigma)$ to $(\rl, \bc(\rl))$. We can represent this probability space as a morphism in $\bstoch$ between the monoidal unit $*$ and $(\Omega, \Sigma)$: that is, a Markov kernel $\mu: * \times \Sigma \rightarrow [0,1]$. Going forward we will write the type signature $* \times \Sigma \rightarrow [0,1]$ as $\Sigma \rightarrow [0,1]$ for convenience.

We can then represent $f$ and $f'$ with their embeddings into $\bstoch$: the Dirac Markov kernels $\delta_{f},\delta_{f'}$. If we compose $\delta_f$ and $\mu$ in $\bstoch$, we form a new probability measure $(\delta_f \circ \mu): \bc(\rl) \rightarrow [0,1]$, which is the pushforward measure $f_{*}\mu$ of $\mu$ along $f$. 

We now have a hint of how we can reason about the independence or dependence of random variables in $\bstoch$. First, consider the probability measure:
\begin{gather*}
    (\delta_{f} \circ \mu) \otimes (\delta_{f'} \circ \mu): \bc(\rl \times \rl) \rightarrow [0,1]
\end{gather*}
where for $\sigma \times \sigma' \in \bc(\rl \times \rl)$:
%
% NOTE: We write \rl x \rl here rather than \rl^2 because we are working in BorelStoch, which is NOT strict monoidal. Therefore \rl x \rl and \rl^2 are isomorphic but not necessarily equivalent. 
%
\begin{align*}
    \left[(\delta_{f} \circ \mu) \otimes (\delta_{f'} \circ \mu)\right](\sigma \times \sigma') = \\
    \left[\int_{\omega \in \Omega} \delta_{f}(\omega,\sigma) d\mu \right]
    \left[\int_{\omega \in \Omega} \delta_{f'}(\omega,\sigma') d\mu \right] = \\
    f_{*}\mu(\sigma)f'_{*}\mu(\sigma').
\end{align*}
This is simply the product measure over $(\rl \times \rl, \bc(\rl \times \rl))$ of the probability measures $(\delta_{f} \circ \mu)$ and $(\delta_{f'} \circ \mu)$ over $(\rl, \bc(\rl))$. It is completely determined by the marginal distributions of $f$ and $f'$ over the probability space $(\Omega, \Sigma, \mu)$, and it is agnostic to the independence or dependence structure of $f$ and $f'$. The reason for this is that the measure $\mu$ is essentially ``duplicated'', and the random variables $f$ and $f'$ are not actually compared over the same probability space.

In contrast, consider instead the probability measure:
\begin{gather*}
    (\delta_{f} \otimes \delta_{f'}) \circ \cp \circ \mu: \bc(\rl \times \rl) \rightarrow [0,1]
\end{gather*}
where $\cp: \Omega \rightarrow \Omega \otimes \Omega$ is the comonoidal copy map at $\Omega$ in $\bstoch$  \citep{fritz2020synthetic}. We can see that for $\sigma \times \sigma' \in \bc(\rl \times \rl)$:
\begin{gather*}
    \left[(\delta_{f} \otimes \delta_{f'}) \circ \cp \circ\mu\right](\sigma \times \sigma') = 
    \left[\int_{\omega \in \Omega} \delta_{f}(\omega,\sigma)\delta_{f'}(\omega,\sigma') d\mu \right].
    % =
    % %
    % (f \otimes f')_{*}\hat{\mu}(\sigma \times \sigma')
\end{gather*}
% This is the pushforward of $\hat{\mu}$ along $f\otimes f'$. That is, it is
This is the probability measure over $(\rl \times \rl, \bc(\rl \times \rl))$ associated with the joint distribution of the random variables $f$ and $f'$ over $(\Omega, \Sigma, \mu)$.

Therefore, the random variables $f$ and $f'$ are independent over the probability space $(\Omega, \Sigma, \mu)$ if and only if the probability measures $(\delta_{f} \circ \mu) \otimes (\delta_{f'} \circ \mu)$ and $(\delta_f \otimes \delta_f') \circ \cp \circ \mu$ are equal.

\section{The co-Kleisli Construction}

Fong et al. and Cruttwell et al. \cite{fong2019backprop, cruttwell2021categorical} build their characterization of machine learning optimization problems on top of the category of Euclidean spaces and parameterized infinitely differentiable maps between them. Rather than represent the loss function itself categorically, the authors treat it as an externally-provided hyperparameter.

However, in practice the loss function is usually implied by the problem. A common problem statement is as follows: given some parameterized random variable, derive the parameters that maximize the likelihood of some observed data being drawn from the distribution of this random variable. A natural question is therefore whether it is possible to replace the parameterized infinitely differentiable maps in Fong et al.'s \cite{fong2019backprop} construction with parameterized random variables.

A quick note on the category of Euclidean spaces and parameterized infinitely differentiable maps between them: Fong et al. \cite{fong2019backprop} calls this category $\para$ whereas Cruttwell et al. \cite{cruttwell2021categorical} call it $\para(\euc)$ (Definition \ref{definition:para}). We will use Cruttwell et al.'s \cite{cruttwell2021categorical} notation, but we will work with the category $\eucmeas$ (Proposition \ref{proposition:eucmeas}) instead of $\euc$ to make it easier to talk about probabilistic constructions. 

Before moving to $\para(\eucmeas)$, we will start with the category $\eucmeas$ (Proposition \ref{proposition:eucmeas}) of Euclidean spaces and infinitely differentiable maps between them. Our first step will be to replace the morphisms in $\eucmeas$ with stochastic processes, or indexed families of random variables.

To start, note that $(O \times \_): \cb \rightarrow \cb$ is an endofunctor that maps the object $A \in Ob(\cb)$ to to $O \times A$ and maps the morphism $f: A \rightarrow B$ to the morphism $id_{O} \otimes f: O \times A \rightarrow O\times B$. We can now introduce the following definition:

\begin{definition}\label{definition:coklo}
% Product comonad, defined for Cartesian categories http://www.cs.ox.ac.uk/ralf.hinze/WG2.8/28/slides/ralf.pdf
%
% https://bartoszmilewski.com/2017/01/02/comonads/
% 
% NOTE: Comonad is an endofunctor F: C -> C with some other data
%
% NOTE: The CoKleisli category under the comonad F: C -> C has the same objects as C and morphisms in [A,B] are morphisms in FA -> B
%
% NOTE: (O x _): C -> C is probably not a monoidal category so there is probably no simple way to define a monoidal product on CoKl_O(C)
%
For some Cartesian monoidal category $\cb$ and object $O$ in $\cb$, $\coklo(\cb)$ is the co-Kleisli category of $\cb$ under the product comonad $(O \times \_)$.
\end{definition}

The category $\mathbf{CoKl}_{O}(\cb)$ has the same objects as $\cb$ and the morphisms in $\mathbf{CoKl}_{O}(\cb)[A,B]$ are morphisms in $\cb[O\times A, B]$. The identity map at the object $A$ is:
\begin{gather*}
    (\del_{O} \otimes id_{A}): O \times A \rightarrow A
\end{gather*}
where $\del_{O}: O \rightarrow *$ is the unique map from $O$ to the terminal object $*$. The $\mathbf{CoKl}_{O}(\cb)$-composition of the morphisms $f: O \times A \rightarrow B$ and $g: O \times B \rightarrow C$ is:
\begin{gather*}
    (g \circ_{\mathbf{CoKl}_{O}(\cb)} f): O \times A \rightarrow C
    \\
    g \circ_{\mathbf{CoKl}_{O}(\cb)} f = 
    g \circ_{\cb}
    (id_{O} \otimes f) \circ_{\cb}
    (\cp_{O} \otimes id_{A})
\end{gather*}
%
% NOTE: Every Cartesian category has a diagonal http://nlab-pages.s3.us-east-2.amazonaws.com/nlab/show/cartesian+monoidal+category
where $\cp_{O}: O \rightarrow O \times O$ is the copy map (aka diagonal) in $\cb$.

For example, if $\Omega$ is $\rl^n$ for some $n \in \mathbb{N}$, the category $\mathbf{CoKl}_{(\Omega, \bo)}(\eucmeas)$ (which we will hereafter abbreviate $\ceucmeas$, see Table \ref{legend}) has the same objects as $\eucmeas$, and the morphisms between $\rl^a$ and $\rl^b$ are continuously differentiable 
% NOTE: Functions R^n -> R^m that are differentiable are Borel-measurable (https://mathworld.wolfram.com/MeasurableFunction.html)
(and therefore Borel-measurable) functions of the form $f: \Omega \times \rl^a \rightarrow \rl^b$.

In $\ceucmeas$, the identity arrow at $\rl^a$ is the function $f(\omega, x_a) = x_a$ and the composition of $f: \Omega \times \rl^a \rightarrow \rl^b$ and $f': \Omega \times \rl^b \rightarrow \rl^c$ is $(f' \circ f): \Omega \times \rl^a \rightarrow \rl^c$ where for $\omega \in \Omega, x_a \in \rl^a$:
\begin{gather*}
    (f' \circ f)(\omega, x_a) = f'(\omega, f(\omega, x_a)).
\end{gather*}
%
% And the tensor of $f': \Omega \times \rl^c \rightarrow \rl^d$ and $f: \Omega \times \rl^a \rightarrow \rl^b$  is $(f' \otimes f): \Omega \times \rl^a \times \rl^c \rightarrow \rl^b \times \rl^d$ where for $\omega \in \Omega, x_a \in \rl^a, x_c \in \rl^c$:
% %
% \begin{gather*}
%     (f' \otimes f)(\omega, (x_a, x_c)) = (f(\omega, x_a), f'(\omega, x_c)).
% \end{gather*}
%
One important thing to note is that $\omega$ is reused when we compose $f$ and $f'$. This allows us to make the following claim:
\begin{proposition}
For any $\omega \in \Omega$, the identity-on-objects map that sends the function $f: \Omega \times \rl^a \rightarrow \rl^b$ in $\ceucmeas$ to the function $f(\omega, \_): \rl^a \rightarrow \rl^b$ in $\eucmeas$ is a functor $\realize: \ceucmeas \rightarrow \eucmeas$, which we call the realization functor.
\end{proposition}
\begin{proof}
First, if $f$ is the identity map in $\ceucmeas$ then $f(\omega, \_)$ is by definition the identity function. Next, consider $f: \Omega \times \rl^a \rightarrow \rl^b, f': \Omega \times \rl^b \rightarrow \rl^c$ in $\ceucmeas$ and any $x_a \in \rl^a$. Then:
\begin{gather*}
(\realize f' \circ \realize f): \rl^a \rightarrow \rl^c
\\
(\realize f' \circ \realize f)(x_a) =
(f'(\omega, \_) \circ f(\omega, \_))(x_a) =
f'(\omega, f(\omega, x_a)) =
\realize(f' \circ f)(x_a)
\end{gather*}
so composition is preserved.
% Finally, consider $g: \Omega \times \rl^a \rightarrow \rl^b, g': \Omega \times \rl^c \rightarrow \rl^d$ in $\ceucmeas$ and any $x_a \in \rl^a, x_c \in \rl^c$. Then:
% %
% \begin{gather*}
% %
% (\realize g \otimes \realize g'): \rl^{a+c} \rightarrow \rl^{b+d}
% %
% \\
% %
% (\realize g \otimes \realize g')(x_a, x_c) =
% % %
% % (g(\omega, \_) \otimes g'(\omega, \_))(x_a, x_c) =
% %
% (g(\omega, x_a), g'(\omega, x_c)) =
% %
% \realize(g \otimes g')(x_a, x_c)
% \end{gather*}
% %
% so the monoidal tensor is preserved.
%
\end{proof}
Given a probability measure $\mu: \bo \rightarrow [0,1]$, we can think of $\ceucmeas$ as a category of differentiable stochastic processes defined on the probability space $\baseprob$. One particularly important kind of stochastic process is a Levy Process.
\begin{definition}
A Levy Process is a one-dimensional stochastic process $f: \Omega \times \rl \rightarrow \rl$ defined on the probability space $\baseprob$ such that:
\begin{itemize}
    \item $f(\_, 0) = 0$ almost surely.
    \item For $t_d > t_c > t_b > t_a \in \rl$, the random variables $f(\_, t_b) - f(\_, t_a)$ and $f(\_, t_d) - f(\_, t_c)$ are independent.
    \item For $t_b > t_a \in \rl$, the random variables $f(\_, t_b) - f(\_, t_a)$ and $f(\_, t_{b}-t_{a})$ have the same distribution.
    \item For any $\omega \in \Omega$ the function $f(\omega, \_)$ is continuous.
\end{itemize}
\end{definition}
We can view Levy processes as continuous-time generalizations of random walks, or as Brownian motions with drift.
\begin{definition}
A subordinator is a non-decreasing Levy Process. That is, if $f$ is a subordinator then for any fixed $\omega \in \Omega$ the function $f(\omega, \_)$ is non-decreasing.
\end{definition}
Since subordinators are closed under composition we can show the following:
\begin{proposition}
Continuously differentiable subordinators form a single-object  subcategory of $\ceucmeas$ at $(\rl, \bc(\rl))$.
\end{proposition}
\begin{proof}

First, note that the identity arrow on $\rl$ is trivially a subordinator. Next, suppose $f$ and $g$ are subordinators. By Lalley et al. \cite{lalley2007levy} we have that $g \circ f$ is a Levy Process. Since both $f$ and $g$ are non-decreasing, for $t_2 > t_1$ we have for any fixed $\omega \in \Omega$ that:
\begin{gather*}
    g(\omega, f(\omega, t_2)) > g(\omega, f(\omega, t_1)).
\end{gather*}
Therefore, $g \circ f$ is a subordinator as well.
\end{proof}

% \section{Random Variables and Transformations}
% Consider the category of random variables and transformations. The objects in this category are pairs $(\rl^a, f)$ where $f$ is a $\rl^a$-valued random variable over the probability space $\baseprob$, and the morphisms between the random variables $f$ and $f'$ are random variable transformations, or functions $h$ such that $h(f(\omega)) = f'(\omega)$.

% Note that random variable transformations are distinct from measure-preserving transformations, which exist between random variables with the same distribution. Even if $f'(\omega) = h(f(\omega))$, it is possible that the distributions of $f$ and $f'$ are distinct. Similarly, even if $f$ and $f'$ have identical distributions, there will not be any random variable transformation between them if $f'$ is not a function of $f$ (e.g. if $f$ and $f'$ are independent and non-degenerate).

% This category is a subcategory of the coslice category $1/\ceucmeas$. A random variable transformation is simply a $\ceucmeas$ arrow that ignores its first argument.

\subsection{Independence and Dependence in $\ceucmeas$}\label{independenceincokl}

Since all of the stochastic processes in $\ceucmeas$ are defined over the same probability space $\baseprob$,
% the composition and tensor use a single source of randomness,
there is a major difference between how $\ceucmeas$ and $\bstoch$ represent independence and dependence. Given the arrows $f: \Omega \times \rl^a \rightarrow \rl^b$ and $f': \Omega \times \rl^c \rightarrow \rl^d$ in $\ceucmeas$ and the vectors $x_a\in\rl^a, x_c\in\rl^c$, the random variables $f(\_, x_a)$ and $f'(\_, x_c)$ may be either dependent or independent.

In order to see how this differs from the situation in $\bstoch$, recall that the pushforward of $\mu$ along the stochastic process $f: \Omega \times \rl^a \rightarrow \rl^b$ is a mapping from $\ceucmeas$ to $\bstoch$ such that for $x_a \in \rl^a, \sigma_b \in \bc(\rl^b)$:
\begin{gather*}
    f_{*}\mu: \rl^a \times \bc(\rl^b) \rightarrow [0,1]
    \\
    f_{*}\mu(x_a, \sigma_b)  = 
    f(\_, x_a)_{*}\mu(\sigma_b) =
    \mu(f(\_, x_a)^{-1}(\sigma_b)) =
    \int_{\omega  \in \Omega} \delta(f(\omega, x_a), \sigma_b) d\mu.
\end{gather*}
However, this mapping does not form a functor. We see that for $f: \Omega \times \rl^a \rightarrow \rl^b$, $f': \Omega \times \rl^b \rightarrow \rl^c$, $x_a \in \rl^a, \sigma_c \in \mathcal{B}(\rl^c)$:
\begin{gather*}
    \left(f' \circ f\right)_{*}\mu: \rl^a \times \bc(\rl^c) \rightarrow [0,1]
\end{gather*}
\begin{align*}
    \left(f' \circ f\right)_{*}\mu(x_a, \sigma_c) =
    \\
    \mu((f' \circ f)(\_, x_a)^{-1}(\sigma_c)) =
    \\
    \int_{\omega  \in \Omega}
    \delta((f'(\omega, f(\omega, x_a)), \sigma_c) d\mu =
    %
    % Apply the identity in https://en.wikipedia.org/wiki/Dirac_measure that f(y) = \int_{x \in X} f(x) d\delta(y, \_) to represent 
    % f(y_b) = \delta(f'(\omega, y_b), \sigma_c)
    % as
    % f(y_b) = \int_{x_b \in Rb} \delta(f'(\omega, x_b), \sigma_c) d\delta(y_b, \_)
    % where y_b = f(\omega, x_a)
    %
    \\
    \int_{\omega  \in \Omega}
    \left(
    \int_{x_b \in \rl^{b}} 
    \delta(f'(\omega, x_b), \sigma_c)
    d\delta(f(\omega, x_a), \_)
    \right)
    d\mu =
    \\
    %
    % Simplify
    \int_{x_b \in \rl^{b}} 
    \int_{\omega \in \Omega}
    \delta(f'(\omega, x_b), \sigma_c) \  d\delta(f(\omega, x_a), \_)  \ d\mu
\end{align*}
whereas:
\begin{gather*}
    \left[f'_{*}\mu \circ f_{*}\mu\right]: \rl^a \times \bc(\rl^c) \rightarrow [0,1]
\end{gather*}
\begin{align*}
    \left[f'_{*}\mu \circ f_{*}\mu\right](x_a, \sigma_c) =
    \\
    %
    % Stoch composition
    \int_{x_b \in \mathbb{R}^b}
    \mu(f'(\_, x_b)^{-1}(\sigma_c))\ 
    d\mu(f(\_, x_a)^{-1}(\_)) = 
    \\
    %
    % Expand the definition of the pushforward
    \int_{x_b \in \mathbb{R}^b}
    \left( \int_{\omega \in \Omega}
    \delta(f'(\omega, x_b), \sigma_c) d\mu \right)
        \left( \int_{\omega \in \Omega}
        d\delta(f(\omega, x_a), \_)d\mu \right).
    % \int_{x_b \in \rl^b} \left[f'(\_, x_b)_{*}\mu(\sigma_c)\right] df(\_, x_a)_*\mu
\end{align*}
These are not necessarily equivalent if the random variables $f'(\_,x_b), x_b\in\rl^b$ are not independent of the random variable $f(\_,x_a)$.

The reason for this mismatch comes down to the fact that composition in $\bstoch$ is based on the Markov property, whereas composition in $\ceucmeas$ is not. In the next Section we will define a new category of stochastic processes that exhibits this independence behavior.

\begin{center}\label{legend}
\begin{table}[h!]
    \centering
    \begin{tabular}{c|c}
        \hline
        Shorthand Name & Full Name \\ [0.5ex] 
        \hline
        \hline
        $\ceucmeas$ & $\mathbf{CoKl}_{(\Omega, \bo)}(\eucmeas)$\\
        \hline
        $\peuc$ & $\parao(\eucmeas)$\\
        \hline
        % $\df$  & $\parao(\para(\eucmeas))$\\
        % %
        % \hline
    \end{tabular}
    \caption{We introduce several compositional constructions for building new categories in this section. These can produce unwieldy names, so for readability we have abbreviated some of them here.}
    \label{abbreviations}
\end{table}
\end{center}

\section{The Parameterization Construction}

In order to reason about the behavior of a system of stochastic processes, it is useful to
% modify the stochastic processes' structure and
study them in a simpler setting. There are two simple ways to do this: take pushforwards and study stochastic processes as Markov kernels, or take expectations and study stochastic processes as functions. In order to make these lines of study rigorous, we first need to establish the functoriality of these transformations. To this end, over the next few Sections we build a new category of stochastic processes in which the map $f \rightarrow f_{*}\mu$ is functorial. In Section \ref{section:expectation-composition} we will explore the functoriality of the expectation.

In order to elevate the pushforward to a functor, we need to modify the definition of how stochastic processes compose. Unlike in $\ceucmeas$, where we treat all stochastic processes as if they were defined over the same probability space, the category in this section will consist of stochastic processes defined over different, non-interacting probability spaces. The composition
% or tensor
of two stochastic processes in this new category will produce a stochastic process over the product of those processes' associated probability spaces. This will allow us to treat all of the stochastic processes in this category as if they were mutually independent. 

We note that this strategy of expanding the probability space each time we introduce a new source of randomness is commonly used by probability theorists \citep{tao, Bill86, ash1975topics}.

\subsection{An Subcategory of $\parac$}
% We will begin by slightly modifying the $\para$ construction from \cite{gavranovic2019compositional, cruttwell2021categorical} (Definition \ref{definition:para}). 

Consider the small symmetric strict monoidal categories $\cb$ and $\db$ such that there exists a faithful identity-on-objects strict monoidal functor $\iota: \db \hookrightarrow \cb$. That is, we can think of $\db$ as a subcategory of $\cb$.

% Then write $(\_ \otimes A) \circ \iota: \db \hookrightarrow \cb$ to denote the functor that sends the object $A'$ in $\db$ to $A' \otimes A$ in $\cb$ and the morphism $f: A \rightarrow B$ in $\db$ to...
% Also write $c_B: \db \rightarrow \cb$ for the constant functor that sends all objects in $\db$ to $B$.

\begin{proposition}\label{proposition:parad}
For the small strict symmetric monoidal categories $\cb$ and $\db$ equipped with a faithful identity-on-objects strict monoidal functor $\iota: \db \hookrightarrow \cb$ we can form a subcategory $\parad(\cb)$ of $\para(\cb)$ (Definition \ref{definition:para}) in which the morphisms in $\parad(\cb)[A,B]$ are pairs $(P, f)$ where $P$ is an object in the image of $\iota$ and $f: P \otimes A \rightarrow B$ is a morphism in $\cb$.
% the set of objects in the comma category $(\_ \otimes A) \circ \iota \downarrow c_B$. That is, the morphisms between $A$ and $B$ in $\parad(\cb)$ are
% limited to morphisms of the form $P \otimes A \rightarrow B$ in $\para(\cb)[A,B]$, where $P$ is an object in $\db$.
\end{proposition}
\begin{proof}
In order to show that $\parad(\cb)$ is a subcategory of $\para(\cb)$ we simply need to show that $\parad(\cb)$ is closed under composition and contains all identities.

To start, note that  $\parad(\cb)$ contains all identities. Since $\db$ is a strict monoidal category we can write the signature of the identity arrow $id_A: A \rightarrow A$ in $\db$ as $id_A: *_{\db} \otimes A \rightarrow A$ where $*_{\db}$ is the monoidal unit in $\db$. Therefore $id_A$ is an arrow in $\parad(\cb)$. This arrow is trivially the $\parad(\cb)$-identity at $A$. 

Next, note that $\parad(\cb)$ is closed under composition. Consider the morphisms:
\begin{gather*}
    f_1: P_1 \otimes A \rightarrow B
    \qquad
    f_2: P_2 \otimes A \rightarrow B
\end{gather*}
where $P_1, P_2$ are objects in the image of $\iota$. Since $\iota$ is identity on objects and strict monoidal it must be that $P_1 \otimes P_2$ is an object in $\db$. Therefore:
\begin{gather*}
    f_2 \circ_{\para(\cb)} f_1: P_2 \otimes P_1 \otimes A \rightarrow B
\end{gather*}
is a morphism in $\parad(\cb)$.
\end{proof}

\subsection{A Category of Parametric Measurable Maps}
In this Section, we will use the $\parad$ construction (Proposition \ref{proposition:parad}) to build a new category of stochastic processes over which the mapping $f \rightarrow f_{*}\mu$ is functorial. In this category composition will have the same independence structure that it has in $\stoch$.
 
\subsubsection{Lawvere Parameterization}

% Para is close to a comonad: https://mathoverflow.net/questions/153362/it-looks-so-cokleisli-but-its-not-what-is-it?fbclid=IwAR0sglYGuHT9WeKsgAdv0X2yryLLDI9NBMjXmsvFdaL5na2cuvS3Ohf5Ifs#comment393181_153362

We begin with the following definition:
\begin{definition}
% NOTE: Any Cartesian monoidal category is symmetric monoidal
Suppose $\cb$ is a strict Cartesian monoidal category, $O^{*}$ is a Lawvere theory with generating object $O$, and $\iota$ is a faithful identity-on-objects strict monoidal functor $\iota: O^{*} \hookrightarrow \cb$. Then $\para_{O^{*}}(\cb)$ is a Lawvere parameterization of $\cb$.
\end{definition}

% TODO: Is there anything useful about this information?
% Let's first note that if $\cb$ is an additive category, then  $\para_{O^{*}}(\cb)$ is the category $(\cb^{op} / O \otimes \_)^{op}$ where $\cb^{op} / O \otimes \_$ is the Orbit category of the automorphism $(O \otimes \_)$ (section 7 of Tabuada \cite{tabuada2011chow}). This implies that the opposite projection functor $\pi^{op}: \cb \rightarrow \para_{O^{*}}(\cb)$ is symmetric monoidal.

Note that the objects in $O^{*}$ are of the form $O \times O \times \cdots \times O$. When the tensor is repeated $n$ times we will write this as $O^n$. We also write $O^0$ for the monoidal unit $*$. For any strict Cartesian monoidal category $\cb$ with a Lawvere parameterization we can define a mapping:
\begin{gather*}
    \copyfunctor: \para_{O^{*}}(\cb) \rightarrow \coklo(\cb)
\end{gather*}
This mapping acts as identity-on-objects and sends the arrow $f: O^n \times A \rightarrow B$ in $\para_{O^{*}}(\cb)$ to the following arrow in  $\coklo(\cb)$:
\begin{gather*}
    f \circ_{\cb} (\cp_{O}(n) \otimes id^{\cb}_A):  O \times A \rightarrow B.
\end{gather*}
%
% NOTE: by https://ncatlab.org/nlab/show/cartesian+monoidal+category the comonoid copy map always exists for any cartesian monoidal categoty
Where $id^{\cb}_A$ is the identity arrow on $A$ in $\cb$ and $\cp_{O}(n)$ is the $n-1$ repeated application of the
% NOTE: All cartesian categories have comonoids at every object https://ncatlab.org/nlab/show/cartesian+monoidal+category. Therefore this map \cp_n exists because C is Cartesian monoidal which means that there is a unique terminal object and this object is the monoidal unit
copy map $O \rightarrow O \times O$ in $\cb$. That is:
\begin{itemize}
    \item $\cp_{O}(3): O \rightarrow O \times O\times O$ is the double application of the copy map $(id_{O}\otimes \cp_{O}) \circ \cp_{O}$.
    \item $\cp_{O}(2): O \rightarrow O \times O$ is just the copy map $\cp_O$.
    \item $\cp_{O}(1): O \rightarrow O$ is the identity map $id_O$ in $\cb$. 
    \item $\cp_{O}(0): O \rightarrow *$ is $\del_{O}$, the unique map from $O$ to the terminal object $*$.
\end{itemize}, 
\begin{proposition}\label{CopyFunctor}
Suppose $\cb$ is a Cartesian monoidal category and $O$ is an object in $\cb$. Then $\copyfunctor: \para_{O^{*}}(\cb) \rightarrow \coklo(\cb)$ is a full identity-on-objects functor.
\end{proposition}
\begin{proof}
First, we note that $\copyfunctor$ is identity-on-objects by definition.

Next, consider any objects $A,B$ in $\cb$ and any arrow $f: O \times A \rightarrow B$ in $\coklo(\cb)$ (Definition \ref{definition:coklo}). Then $f$ is also an arrow in $\para_{O^{*}}(\cb)$ and $\copyfunctor$ maps $f$ to:
\begin{gather*}
   f \circ_{\cb} (\cp_{O}(1) \otimes id^{\cb}_{A}) = f. 
\end{gather*}
Therefore $\copyfunctor$ is full. 

Next, since $\cb$ is strict monoidal we have
\begin{gather*}
    A = * \times A = O^0 \times A
\end{gather*}
which implies that $\copyfunctor$ maps the $id_A: A \rightarrow A$ arrow in $\para_{O^{*}}(\cb)$ to the arrow:
\begin{gather*}
    id_A \circ_{\cb} (\cp_{O}(0) \otimes id^{\cb}_A): O \times A \rightarrow A
\end{gather*}
which is the identity arrow in $\coklo(\cb)$. Therefore, $\copyfunctor$ preserves identity morphisms.

Next, we will show $\copyfunctor$ preserves composition. Suppose $f: O^m \times A \rightarrow B$ and $f': O^n \times B \rightarrow C$ are arrows in $\para_{O^{*}}(\cb)$:
\begin{gather*}
    (\copyfunctor f' \circ \copyfunctor f): O \otimes A \rightarrow C
\end{gather*}

\begin{align*}
    (\copyfunctor f' \circ_{\coklo(\cb)} \copyfunctor f) = \\
    %
    % Note that the left and right groups are composed in CoKl
    \left(f' \circ_{\cb}
    (\cp_{O}(n) \otimes id^{\cb}_B)\right)
    \circ_{\coklo(\cb)}
    \left(f \circ_{\cb}
    (\cp_{O}(m) \otimes id^{\cb}_A)\right) = \\
    %
    % Apply CoKl composition
    \left(f' \circ_{\cb}
    (\cp_{O}(n) \otimes id^{\cb}_B)\right)
    \circ_{\cb}
    (id^{\cb}_{O} \otimes (f \circ_{\cb}
    (\cp_{O}(m) \otimes id^{\cb}_A)))
    \circ_{\cb}
    (\cp_{O} \otimes id^{\cb}_A) = \\
    %
    % Collapse identities
    f' \circ_{\cb}
    (\cp_{O}(n) \otimes (f \circ_{\cb}
    (\cp_{O}(m) \otimes id^{\cb}_A)))
    \circ_{\cb}
    (\cp_{O} \otimes id^{\cb}_A) =
    \\
    %
    % Move the second application of the copy to earlier in the chain
    f' \circ_{\cb}
    (id_{O^n} \otimes (f \circ_{\cb}
    (\cp_{O}(m) \otimes id^{\cb}_A)))
    \circ_{\cb}
    (\cp_{O}(n+1) \otimes id^{\cb}_A) =
    \\
    % 
    % Move the first application of the copy to earlier in the chain
    f'
    \circ_{\cb}
    (id_{O^n} \otimes f)
    \circ_{\cb}
    (\cp_{O}(n+m) \otimes id^{\cb}_A) = \\
    %
    % 
    % Composition in Para
    (f' \circ_{\para_{O^{*}}(\cb)} f)
    \circ_{\cb}
    (\cp_{O}(n+m) \otimes id^{\cb}_A) = \\
    \copyfunctor(f' \circ_{\para_{O^{*}}(\cb)} f).
\end{align*}
%
% Finally, we will show that $\copyfunctor$ preserves tensor. Suppose $f: O^m \times A \rightarrow B$ and $f': O^n \times C \rightarrow D$ are arrows in $\para_{O^{*}}(\cb)$:
% %
% \begin{gather*}
%     (\copyfunctor f' \otimes \copyfunctor f): O^{n+m} \otimes A \otimes C \rightarrow B \otimes D
% \end{gather*}
% %
% \begin{align*}
%     (\copyfunctor f' \otimes \copyfunctor f) = \\
%     %
%         \left[f' 
%         \circ_{\cb}
%         (\cp_{O}(n) \otimes id^{\cb}_C) \right]
%     \otimes
%         \left[f 
%         \circ_{\cb}
%         (\cp_{O}(m) \otimes id^{\cb}_A)\right] = \\
%     %
%     (f' \otimes f)
%     \circ_{\cb}
%         (\cp_{O}(n+m)
%             \otimes
%         id^{\cb}_{C \otimes A})
%     = \\
%     %
%     \copyfunctor (f' \otimes f).
% \end{align*}
\end{proof}

\subsection{Applying $\para$ to $\eucmeas$}

Now suppose we have a probability space $\baseprob$ where $\Omega$ is $\rl^k, k \in \mathbb{N}$. We can form the Lawvere theory $(\Omega, \bc(\Omega))^{*}$ with generating object $(\Omega, \bo)$ and tuples:
\begin{gather*}
    (\Omega, \bc(\Omega))^n = (\Omega^n, \bc(\Omega^n))
\end{gather*}
as objects. We can also form the faithful identity-on-objects strict monoidal functor as the inclusion:
\begin{gather*}
    \iota: \obo \hookrightarrow \eucmeas
\end{gather*}
Then for any:
\begin{gather*}
    (\Omega^n, \bc(\Omega^n)) \in \obo
\end{gather*}
we can create the probability space $\baseprobn$ where $\mu^n$ is the product measure:
\begin{gather*}
    \mu^n: \bc(\Omega^n) \rightarrow [0,1]
    \\
    \mu^n(\sigma_1 \times \sigma_2 \times \cdots \times \sigma_n) = \mu(\sigma_1)\mu(\sigma_2)\cdots\mu(\sigma_n).
\end{gather*}
\begin{definition}
We can apply $\parao$ to $\eucmeas$ to form the Lawvere parameterization $\parao(\eucmeas)$, which we will hereafter abbreviate $\peuc$.
\end{definition}
Intuitively, $\peuc$ allows us to reason about probabilistic relationships in terms of measurable functions rather than probability measures.

Next, by Proposition \ref{CopyFunctor}, we have an identity-on-objects functor, $\copyfunctor$, from $\peuc$ to $\ceucmeas$. Let's drill deeper into this relationship. We can view an arrow of the form $f: \Omega^n \times \rl^a \rightarrow \rl^b$ in $\peuc$ as a stochastic process over $\baseprobn$. However, unlike in $\ceucmeas$, if we compose
% or tensor
$f$ with another arrow in $\peuc$, we do not get another stochastic process over $\baseprobn$. Instead, we get a stochastic process over some other probability space. Intuitively, we can think of the stochastic processes in $\peuc$ as being defined over different, non-interacting probability spaces.

Now given some arrow $f: \Omega^n \times \rl^a \rightarrow \rl$ in $\peuc$ and $x_a \in \rl^a$, the measurable function $f(\_, x_a)$ is a real-valued random variable over the probability space $\baseprobn$. The pushforward of $\mu^n$ along this random variable $f(\_,x_a)_{*}\mu^{n}(\_): \bc(\rl) \rightarrow [0,1]$ is then a probability measure over the space $(\rl, \bc(\rl))$. 
% NOTE: these random variables/stochastic processes are not directly comparable because they are not defined over the same probability space

In general, we can extend this pushforward procedure to define a mapping between parametric families of measurable maps and Markov kernels. Given some $f: \Omega^n \times  \rl^a \rightarrow \rl^b$ we can define:
\begin{gather*}
\push f: \rl^a \times \bc(\rl^b) \rightarrow [0,1]
\end{gather*}
where for $x_a \in \rl^a, \sigma_b \in \bc(\rl^b)$:
\begin{gather*}
    \push f(x_a, \sigma_b)  = f(\_, x_a)_{*}\mu^n(\sigma_b) =
    \int_{\omega_n  \in \Omega^n} \delta(f(\omega_n, x_a), \sigma_b) d\mu^n.
\end{gather*}
\begin{proposition}\label{PushFunctor}
The mapping $\push$ that takes a parametric family $f: \Omega^n \times \rl^a \rightarrow \rl^b$ of measurable maps to the Markov kernel $f_{*}\mu^n$ is an identity-on-objects
% strict monoidal
functor from $\peuc$ to $\bstoch$.
\end{proposition}
\begin{proof}
In this proof we rely on the following property of product measures, which holds by Fubini's theorem when $\Omega=\rl^k$:
\begin{gather*}
    % http://theanalysisofdata.com/probability/F_5.html
    %
    \int_{(\omega_n,\omega_m)  \in \Omega^n \times \Omega^m}
    f(\omega_n, \omega_m) d\mu^{n+m} =
    \int_{\omega_n \in \Omega^n} \int_{\omega_m \in \Omega^m}
    f(\omega_n, \omega_m) d\mu^{n}d\mu^{m}
\end{gather*}

We first note that since the objects in $Ob(\peuc) = Ob(\eucmeas)$ are the standard Borel measurable spaces $(\rl^n, \bc(\rl^n))$ it must be that $\push$ maps objects in $\peuc$ to objects in $\bstoch$.

Next, note that for any $\rl^a$, $\push$ maps the identity at $\rl^a$ in $\peuc$ to the identity at $\rl^a$ in $\bstoch$ since:
\begin{align*}
    \push id_{\rl^a}(x_a, \sigma_a) = \\
    \int_{\omega_n  \in \Omega^n} \delta(id_{\rl^a}(\omega_n, x_a), \sigma_a) d\mu^n = \\ 
    \int_{\omega_n  \in \Omega^n} \delta(x_a, \sigma_a) d\mu^n = \\
    \delta(x_a, \sigma_a)
\end{align*}
Next, we will demonstrate that $\push$ preserves composition. Suppose we have some:
\begin{gather*}
    f: \Omega^n \times \rl^a \rightarrow \rl^b
    \qquad
    f': \Omega^m \times \rl^b \rightarrow \rl^c
    \\
    x_a \in \rl^a
    \qquad
    \sigma_c \in \bc(\rl^c)
\end{gather*}
Then we can write:
\begin{gather*}
  \push\left(f' \circ f\right): \rl^a \times \bc(\rl^c) \rightarrow [0,1]
\end{gather*}
\begin{align*}
  \push\left(f' \circ f\right)(x_a, \sigma_c) = \\
  %
  % Definition of \push
  \int_{(\omega_m,\omega_n)  \in \Omega^m \times \Omega^n}
  \delta((f' \circ f)((\omega_m, \omega_n), x_a), \sigma_c) d\mu^{n+m} = \\
  % 
  % Since \mu is a product measure, we can separate out the \omega_m and \omega_m. Also apply the definition of \peuc composition.
  \int_{\omega_m \in \Omega^m}
  \int_{\omega_n \in \Omega^n}
  \delta(f'(\omega_m, f(\omega_n, x_a)), \sigma_c) d\mu^n d\mu^m = \\
  %
  % Apply identity in https://en.wikipedia.org/wiki/Dirac_measure that:
  %   f(y) = \int_{x \in X} f(x) d\delta(y, \_)
  %
  \int_{\omega_m \in \Omega^m}
  \int_{\omega_n \in \Omega^n}
  \int_{x_b \in \rl^b}
  \delta(f'(\omega_m, x_b), \sigma_c) d\delta(f(\omega_n, x_a), \_) d\mu^n d\mu^m  = \\
  %
  % Apply Fubini to rearrange integral
  \int_{x_b \in \rl^b}
  \int_{\omega_m \in \Omega^m}
  \int_{\omega_n \in \Omega^n}
  \delta(f'(\omega_m, x_b), \sigma_c) d\delta(f(\omega_n, x_a), \_) d\mu^n d\mu^m  = \\
  %
  % Pass \omega_n through the integral
  \int_{x_b \in \rl^b}
  \left[
  \int_{\omega_m \in \Omega^m}
  \delta(f'(\omega_m, x_b), \sigma_c)
  d\mu^m
  \right]
  \left[
  \int_{\omega_n \in \Omega^n}
  d\delta(f(\omega_n, x_a), \_)
  d\mu^n
  \right] =
  \\
  %
  % apply https://mathoverflow.net/questions/412896/does-the-radon-nikodym-derivative-commute-with-integration
  \int_{x_b \in \rl^b}
  \left[
  \int_{\omega_m \in \Omega^m}
  \delta(f'(\omega_m, x_b), \sigma_c)
  d\mu^m
  \right]
  d\left[
  \int_{\omega_n \in \Omega^n}
  \delta(f(\omega_n, x_a), \_)
  d\mu^n
  \right] =
  \\
  %
  % Definition of \push
  \int_{x_b \in \rl^b}
  [\push f'](x_b, \sigma_c)
  \ 
  d[\push f](x_a, \_) = \\
  %
  % Definition of \stoch composition
  (\push f' \circ \push f)(x_a, \sigma_c).
\end{align*}
%
% Finally, we will demonstrate that $\push$ preserves tensor. Suppose we have some:
% %
% \begin{gather*}
%     f: \Omega^n \times \rl^a \rightarrow \rl^b 
%     \qquad
%     f': \Omega^m \times \rl^c \rightarrow \rl^d
%     \\
%     x_a \in \rl^a\qquad
%     x_c \in \rl^c
%     \\
%     \sigma_d \in \bc(\rl^d)\qquad
%     \sigma_b \in \bc(\rl^b)
% \end{gather*}
% %
% Then we can write:
% \begin{gather*}
%   \push\left(f' \circ f\right): \rl^{c+a} \times \bc(\rl^{d+b}) \rightarrow [0,1]
% \end{gather*}
% %
% \begin{align*}
%   %
%   \push\left(f' \otimes f\right)((x_{c}, x_{a}), (\sigma_{d} \times \sigma_{b})) = \\
%   %
%   % Definition of \push, expand the integrals
%   \int_{\omega_m \in \Omega^m} \int_{\omega_n\in \Omega^n}
%   \delta((f' \otimes f)((\omega_m, \omega_n), (x_{c}, x_{a})), (\sigma_{d} \times \sigma_{b}))
%   d\mu^{m} d\mu^{n} = \\
%   %
%   % Apply the definition of \peuc tensor and then separate out the delta into two deltas, based on the definition of the dirac delta?
%   %
%   \int_{\omega_m \in \Omega^m} \int_{\omega_n\in \Omega^n}
%   \delta(f'(\omega_m, x_c), \sigma_d)
%   \delta(f(\omega_n, x_a), \sigma_b)
%   d\mu^{m} d\mu^{n} = \\
%   %
%   % Move the integral
%   \int_{\omega_m \in \Omega^m}
%   \delta(f'(\omega_m, x_c), \sigma_d)
%   d\mu^{m}
%   \int_{\omega_n \in \Omega^n}
%   \delta(f(\omega_n, x_a)), \sigma_b)
%   d\mu^{n} = \\
%   %
%   % Definition of \push
%   (\push f')(x_c, \sigma_d)(\push f)(x_a, \sigma_b) = \\
%   %
%   % Definition of \otimes in \stoch
%   (\push f' \otimes \push f)((x_a, x_c), (\sigma_{d} \otimes \sigma_{b})).
% \end{align*}

\end{proof}

\subsection{Composition Experiments}

We can express the difference between composition in $\peuc$ and $\ceucmeas$ with a simple experiment using the numpy \citep{harris2020array} and scipy \citep{2020SciPy-NMeth} libraries.

Consider the probability space $\base$ where $\Omega=[0,1]$ and $\mu$ is the uniform measure. We can represent samples from this with the following:
\begin{lstlisting}[language=Python]
import numpy as np
omega_samples = np.random.random(10000)
\end{lstlisting}
Now consider the following stochastic process $f: \Omega \times \rl \rightarrow \rl$ over $\base$:
\begin{lstlisting}[language=Python]
from scipy.stats import norm
def f(omega, x):
    normal_points = norm(loc=5 - x, scale=10).ppf(omega)
    return normal_points
\end{lstlisting}
Note that for any $x \in \rl$, the random variable $f(\_, x)$ over has a normal distribution:
\begin{lstlisting}[language=Python]
input_x = 42
f_points = f(omega_samples, input_x)
\end{lstlisting}
% plt.title("Normal Null Hypothesis Pvalue {}".format(normaltest(f_points).pvalue))
% h = plt.hist(f_points, bins=(np.linspace(np.min(f_points) - 1, np.max(f_points) + 1, 20)))
%
\begin{figure}[H]\begin{center}
\includegraphics[width=6cm,height=6cm]{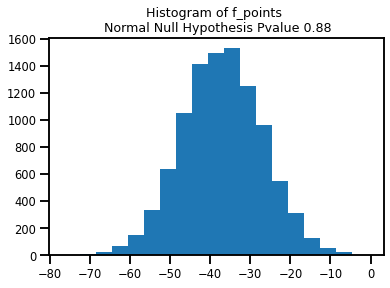}
\label{normal_samples}
\end{center}
\end{figure}
Note that $f$ is an endomorphism on $\rl$ in $\peuc$, so we can take the $\peuc$-composition of $f$ with itself to form the arrow $(f \circ f): \Omega^2 \times \rl \rightarrow \rl$. We write this arrow as:
\begin{lstlisting}[language=Python]
def ff_para(omega1, omega2, x):
    return f(omega2, f(omega1, x))
\end{lstlisting}
Note that $(f \circ f)$ is a stochastic process over the product probability space $(\Omega^2, \mathcal{B}(\Omega^2), \mu^2)$, and that for any $x \in \rl$, the random variable $(f\circ f)(\_, \_, x)$ over $(\Omega^2, \mathcal{B}(\Omega^2), \mu^2)$ has a normal distribution as well:
\begin{lstlisting}[language=Python]
input_x = 42
omega2 = np.random.random((10000, 2))
ff_para_points = ff_para(omega2[:, 0], omega2[:, 1], input_x)
\end{lstlisting}
% plt.title("Normal Null Hypothesis Pvalue {}".format(normaltest(ff_para_points).pvalue))
% h = plt.hist(ff_para_points, bins=(np.linspace(np.min(ff_para_points) - 1, np.max(ff_para_points) + 1, 20)))
%
\begin{figure}[H]\begin{center}
\includegraphics[width=6cm,height=6cm]{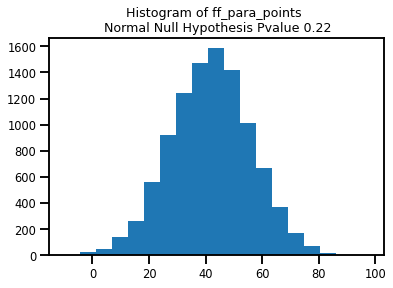}
\label{composed_normal_samples}
\end{center}
\end{figure}
Now recall the functor $\copyfunctor: \peuc \rightarrow \ceucmeas$ from Proposition \ref{CopyFunctor}. This functor acts as identity-on-objects and sends the arrow $f: \Omega^n \times \rl \rightarrow \rl$ in $\peuc$ to the following arrow in  $\ceucmeas$:
\begin{gather*}
    f \circ_{\eucmeas} (\cp_{\Omega}(n) \otimes id_{\rl}):  \Omega \times \rl \rightarrow \rl.
\end{gather*}
We can implement this functor as follows:
\begin{lstlisting}[language=Python]
import inspect
from functools import partial
def CopyFunctor(f):
    def g(omega, x, f=f):
        while len(inspect.getargspec(f).args) > 1:
            f = partial(f, omega)
        return f(x)
    return g
\end{lstlisting}
Note that:
\begin{align*}
    \copyfunctor (f \circ f): \Omega \times \rl \rightarrow \rl
\end{align*}
is a stochastic process over the probability space $\base$. However, unlike $(f \circ f)(\_, \_, x)$, the random variable:
\begin{align*}
    \copyfunctor (f \circ f)(\_, x): \Omega \rightarrow \rl
\end{align*}
is not normal for any fixed $x \in \rl$. Instead, it is constant:
\begin{lstlisting}[language=Python]
input_x = 42
omega = np.random.random(10000)
ff_cokl_points = CopyFunctor(ff_para)(omega, input_x)
\end{lstlisting}
\begin{figure}[H]\begin{center}
\includegraphics[width=6cm,height=6cm]{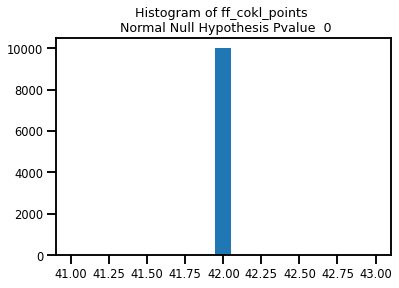}
\label{copied_normal_samples}
\end{center}
\end{figure}

\section{Parameterized Statistical Models}
We have been discussing the arrows in $\peuc$ as parameterized random variables, or stochastic processes, but we can also think of them as $\eucmeas$ arrows with an element of randomness that is dictated by the probability measure $\mu$. One of the primary goals of this work is to replace the domain of Fong et al.'s \cite{fong2019backprop} Backpropagation functor with a probabilistically motivated category over which we can define the error function $er: \rl \times \rl \rightarrow \rl$ through maximum likelihood. Therefore, a natural next step is to extend $\peuc$ to a category in which we can instead think of the arrows as $\para(\eucmeas)$ arrows with an element of randomness added. 

In order to do this, we will replace the stochastic processes in $\peuc$ with parameterized stochastic processes, which we will also refer to as parametric statistical models. That is, the arrows in this category will consist of families of random variables that have two layers of parameterization: one layer acts as the model input (e.g. the independent variable in a linear regression model) and one layer acts as the model parameters (e.g. the slope, intercept and variance terms).

\subsection{The Category $\df$}\label{df}

Given a probability space $\baseprob$ where $\Omega = \rl^k, k\in\mathbb{N}$, any stochastic process $f: \Omega^n \times \rl^a \rightarrow \rl^b$ in $\peuc$ defines a stochastic relationship between values in $\rl^a$ and $\rl^b$. A parametric statistical model is a parameterized family of such relationships. For example, consider a univariate linear regression model $l: \Omega^n \times \rl^3 \times \rl \rightarrow \rl$ where for $\omega_n \in \Omega^n, [a,b,s] \in \rl^3, x \in \rl$:
\begin{gather*}
    l(\omega_n, [a,b,s], x) = ax + b + f_{\mathcal{N}(0, s^2)}(\omega_n)
\end{gather*}
and $f_{\mathcal{N}(0, s^2)}$ is a normally distributed random variable with mean $0$ and variance $s^2$. Any value $[a,b,s] \in \rl^3$ defines the stochastic process, or $\peuc$ arrow:
\begin{gather*}
    l(\_,[a,b,s], \_): \Omega^n \times \rl \rightarrow \rl.
\end{gather*}
For any model input value $x\in\rl$, the function $l(\_,[a,b,s], x)$ is then a random variable defined on the probability space $\baseprobn$. Like with any ordinary univariate linear regression model, this random variable is normally distributed on the real line. 

We can define a category of such models.
%
%\begin{proposition}\label{proposition:df-category}
%Suppose we have a probability space $\baseprob$ where $\Omega$ is $\rl^k, k \in \mathbb{N}$. Consider the product category $(\Omega, \bc(\Omega))^{*} \times \cb$ where $(\Omega, \bc(\Omega))^{*}$ is the Lawvere theory with generating object $(\Omega, \bo)$ and tuples:
%
%\begin{gather*}
%    (\Omega, \bc(\Omega))^n = (\Omega^n, \bc(\Omega^n))
%\end{gather*}
%
%as objects.
%
% We can use the construction from Proposition \ref{proposition:parad} to form the category $\para_{(\Omega, \bc(\Omega))^{*} \times \cb}(\cb)$ that has the same objects as $\cb$, which we rename $\df$ for brevity.
%\end{proposition}
%\begin{definition}
%We can apply $\parao$ to $\para(\eucmeas)$ to form the category:
%
%\begin{gather*}
%    \parao(\para(\eucmeas))
%\end{gather*}
%
%which we rename $\df$ for brevity.
%\end{definition}

\begin{proposition}\label{proposition:df-category}
Suppose we have a probability space $\baseprob$ where $\Omega$ is $\rl^k, k \in \mathbb{N}$. We can define a category $\df$ that has the same objects as $\eucmeas$ (Definition \ref{proposition:eucmeas}) such that the morphisms between $\rl^a$ and $\rl^b$ are $\eucmeas$-morphisms of the form:
\begin{gather*}
f: \Omega^n \times \rl^p \times \rl^a \rightarrow \rl^b
\end{gather*}
The composition of the morphisms:
\begin{gather*}
f_1: \Omega^{n_1} \times \rl^{p_1} \times \rl^a \rightarrow \rl^b
\qquad
f_2: \Omega^{n_2} \times \rl^{p_2} \times \rl^b \rightarrow \rl^c
\end{gather*}
is the morphism:
\begin{gather*}
f_2 \circ f_1: \Omega^{n_2 + n_1} \times \rl^{p_2 + p_1} \times \rl^a \rightarrow \rl^c
\\
(f_2 \circ f_1)(\omega_{n_2}, \omega_{n_1}, x_{p_2}, x_{p_1}, x_{a}) = 
f_2(\omega_{n_2}, x_{p_2}, f_1(\omega_{n_1}, x_{p_1}, x_{a}))
\end{gather*}
\end{proposition}

\begin{proof}
We need to show that $\df$ contains all identities, is closed under composition, and that composition is associative.

To start, note that the identity arrow at the object $\rl^a$ in $\df$ is the function:
\begin{gather*}
id: \Omega^{0} \times \rl^{0} \times \rl^a \rightarrow \rl^a
\\
id(x_a) = x_a
\end{gather*}
and therefore $\df$ contains all identities.

Next, note that the composition of the morphisms:
\begin{gather*}
f_1: \Omega^{n_1} \times \rl^{p_1} \times \rl^a \rightarrow \rl^b
\qquad
f_2: \Omega^{n_2} \times \rl^{p_2} \times \rl^b \rightarrow \rl^c
\\
f_2 \circ f_1: \Omega^{n_2 + n_1} \times \rl^{p_2 + p_1} \times \rl^a \rightarrow \rl^c
\end{gather*}
is in $\df[\rl^a, \rl^c]$ and therefore $\df$ is closed under composition.

Next, consider the morphisms
\begin{gather*}
f_1: \Omega^{n_1} \times \rl^{p_1} \times \rl^a \rightarrow \rl^b
\qquad
f_2: \Omega^{n_2} \times \rl^{p_2} \times \rl^b \rightarrow \rl^c
\qquad
f_3: \Omega^{n_3} \times \rl^{p_3} \times \rl^c \rightarrow \rl^d
\end{gather*}
We have that:
\begin{gather*}
f_3 \circ (f_2 \circ f_1): \Omega^{n_3 + (n_2 + n_1)} \times \rl^{p_3 + (p_2 + p_1)} \times \rl^a \rightarrow \rl^d
\end{gather*}
\begin{align*}
(f_3 \circ (f_2 \circ f_1))((\omega_{n_3}, (\omega_{n_2}, \omega_{n_1})), (x_{p_3}, (x_{p_2}, x_{p_1})), x_{a}) = \\
f_3(\omega_{n_3}, x_{p_3}, (f_2 \circ f_1)((\omega_{n_2}, \omega_{n_1}), (x_{p_2}, x_{p_1}), x_{a})) = \\
f_3(\omega_{n_3}, x_{p_3}, f_2(\omega_{n_2}, x_{p_2}, f_1(\omega_{n_1}, x_{p_1}, x_{a})))
\end{align*}
which is equal to:
\begin{gather*}
(f_3 \circ f_2) \circ f_1: \Omega^{(n_3 + n_2) + n_1} \times \rl^{(p_3 + p_2) + p_1} \times \rl^a \rightarrow \rl^d
\end{gather*}
\begin{align*}
((f_3 \circ f_2) \circ f_1)(((\omega_{n_3}, \omega_{n_2}), \omega_{n_1}), ((x_{p_3}, x_{p_2}), x_{p_1}), x_{a}) = \\
(f_3 \circ f_2)((\omega_{n_3}, \omega_{n_2}), (x_{p_3}, x_{p_2}), f_1( \omega_{n_1}, x_{p_1}, x_{a})) = \\
f_3(\omega_{n_3}, x_{p_3}, f_2(\omega_{n_2}, x_{p_2}, f_1(\omega_{n_1}, x_{p_1}, x_{a})))
\end{align*}
and therefore composition in $\df$ is associative. 
\end{proof}

The name $\df$ derives from the fact that the arrows in this category are \textbf{D}iscriminative and \textbf{F}requentist statistical models (see Table \ref{abbreviations} for a list of all such abbreviations). That is, each arrow operates as if both the parameters and input values are fixed and only the output value is probabilistic. For example, the homset $\df[\rl,\rl]$ includes the linear regression model above. In contrast, generative models and Bayesian models assume a probability distribution over the input and parameter values respectively.

\subsection{Gaussian-Preserving Transformations}
% Can we use the law of unconcious statistician instead: https://en.wikipedia.org/wiki/Law_of_the_unconscious_statistician

\subsubsection{A Subcategory of Gaussian-Preserving Transformations}

\begin{definition}
% NOTE: T needs to be Borel measurable in order to enforce that the composition of T with f is also measurable
A Gaussian-preserving transformation $T:\rl^a \rightarrow \rl^b$ is a Borel measurable function such that for any multivariate normal random variable $f: \Omega^n \rightarrow \rl^a$ defined on the probability space $\baseprobn$, the random variable $(T \circ f): \Omega^n \rightarrow \rl^b$ is multivariate normal and we have:
\begin{gather*}
    \int_{\omega_n \in \Omega^n} T(f(\omega_n)) d\mu =
    T\left(\int_{\omega_n \in \Omega^n}  f(\omega_n) d\mu\right).
\end{gather*}
\end{definition}
% NOTE: Any continuous function on R is Borel measurable https://math.stackexchange.com/questions/3519099/is-this-function-borel-measurable
%
\noindent For example, any linear function is Gaussian-preserving.
% TODO: Extend this to nonlinear functions as well. Maybe also $T(x) = \phi^{-1}(F(x^2))$ where $F$ is the cdf (cumulative distribution function) of the chi-square distribution and $\phi$ is the cdf of the normal distribution \cite{58545} \cite{200407}

Now for some probability space $\baseprob$ where $\Omega = \rl^k, k\in \mathbb{N}$, we can construct a set of $\df$-arrows $\nm$ such that for any $f \in \nm$ with the signature:
\begin{gather*}
    f: \Omega^n \times \rl^p  \times \rl^a \rightarrow \rl^b
\end{gather*}
there exists some map $T: \rl^p \times \rl^a \rightarrow \rl^b$ and multivariate normal random variable $G: \Omega^n \rightarrow \rl^b$ defined on the probability space $\baseprobn$ such that for all $\omega_n \in \Omega^n , x_p \in \rl^p, x_a \in \rl^a$ the map $T(x_p, \_): \rl^a \rightarrow \rl^b$ is a Gaussian-preserving transformation and:
\begin{gather*}
    f(\omega_n, x_p, x_a) = T(x_p, x_a) + G(\omega_n).
\end{gather*}
Note that this includes the univariate linear regression model $l$, as well as the identity arrow, since constant distributions are multivariate normal with variance $0$.

% Note that $\nm$ is closed under the tensor in $\df$, since given the maps:
% %
% \begin{align*}
%     & f': \Omega^m \times \rl^q  \times \rl^c \rightarrow \rl^d \\
%     & f: \Omega^n \times \rl^p  \times \rl^a \rightarrow \rl^b
% \end{align*}
% %
% in $\nm$ and the points:
% %
% \begin{gather*}
%     \omega_m \in \Omega^m \qquad
%     \omega_n \in \Omega^n
%     \\
%     x_q \in \rl^q \qquad
%     x_p \in \rl^p \qquad
%     x_c \in \rl^c \qquad
%     x_a \in \rl^a
% \end{gather*}
% %
% we have that:
% %
% \begin{align*}
%     (f' \otimes f)((\omega_m, \omega_n), (x_q, x_p), (x_c, x_a))))
%     = \\
%     %
%     (T'(x_q, x_c) + G'(\omega_m), T(x_p, x_a) + G(\omega_n)) = \\
%     %
%     (T'(x_q, x_c), T(x_p, x_a)) + (G'(\omega_m), G(\omega_n)).
% \end{align*}
%
Since $\nm$ contains the identity arrows we can construct a useful subcategory of $\df$.
\begin{definition}
$\dfn$ is the category with the same objects as $\df$ and arrows generated by the $\df$-composition of arrows in $\nm$.
\end{definition}

\begin{proposition}\label{proposition:multivariate-normal-closure}
Given any arrow $f: \Omega^n \times \rl^p \times \rl^a \rightarrow \rl^b$ in $\dfn$ and $x_p \in \rl^p, x_a \in \rl^a$, $f(\_,x_p,x_a): \Omega^n \rightarrow \rl^b$ is a multivariate normal random variable defined on the probability space $\baseprobn$.
\end{proposition}
\begin{proof}
We will show that this property holds for the arrows in $\nm$ and that it is preserved by composition.

To begin, note that for any $n,m$, the pushforward of $\mu^m$ along $f: \Omega^{m} \rightarrow \rl^a$ is equivalent to the pushforward of $\mu^{m+n}$ along the following random variable:
\begin{gather*}
    f^l: \Omega^{m+n} \rightarrow \rl^a
    \\
    f^l(\omega_m, \omega_n) = f(\omega_m)
\end{gather*}
We can see this as follows. For any $\sigma_{a} \in \bc(\rl^a)$:
\begin{gather*}
f^{l}_{*}\mu^{m+n}: \bc(\rl^a) \rightarrow [0,1]
\end{gather*}
\begin{align*}
    f^{l}_{*}\mu^{m+n}(\sigma_{a}) = \\
    %
    % Definition of dirac delta measure
    \int_{(\omega_m, \omega_n) \in \Omega^{m+n}}
    \delta(f^{l}(\omega_m, \omega_n), \sigma_a) d\mu^{m+n} = \\
    \int_{\omega_m \in \Omega^m}
    \int_{\omega_n \in \Omega^n}
    \delta(f^{l}(\omega_m, \omega_n), \sigma_a) d\mu^{m}d\mu^{n} = \\
    \int_{\omega_m \in \Omega^m}
    \delta(f(\omega_{m}), \sigma_a) d\mu^{m}
    \int_{\omega_n \in \Omega^n}d\mu^{n} = \\
    f_{*}\mu^{m}(\sigma_{a}).
\end{align*}
By a similar argument we have that the pushforward of $\mu^m$ along $f: \Omega^{m} \rightarrow \rl^a$ is equivalent to the pushforward of $\mu^{n+m}$ along the random variable $f^r(\omega_n, \omega_m) = f(\omega_m)$. 

Next, we note that for any $x_p \in \rl^p, x_a \in \rl^a$ and arrow $f: \Omega^n \times \rl^p \times \rl^a \rightarrow \rl^b \in \nm$, the random variable $f(\_, x_p, x_a):\Omega^n\rightarrow\rl^b$ is multivariate normal and defined on the probability space $(\Omega^n, \bc(\Omega^n), \mu^n)$. This follows from the fact that for $\omega_n \in \Omega^n$:
\begin{gather*}
f(\omega_n, x_p, x_a) = T(x_p, x_a) + G(\omega_n)
\end{gather*}
where $T(x_p, x_a)$ is a constant and $G: \Omega^n\rightarrow\rl^b$ is multivariate normal. Next, we show that for any $x_p \in \rl^p, x_q \in \rl^q, x_a \in \rl^a$, arrow $f': \Omega^m \times \rl^q \times \rl^b \rightarrow \rl^c$ in $\nm$ and arrow $f: \Omega^n \times \rl^p \times \rl^a \rightarrow \rl^b$ in $\df$ such that the random variable $f(\_, x_p, x_a): \Omega^n\rightarrow \rl^b$ is multivariate normal, the random variable:
\begin{gather*}
    (f' \circ f)(\_, (x_q, x_p), x_a): \Omega^{m+n} \rightarrow \rl^b
\end{gather*}
is multivariate normal over $(\Omega^{m+n}, \bc(\Omega^{m+n}), \mu^{m+n})$ since:
\begin{align*}
      (f' \circ f)((\omega_m, \omega_n), (x_q, x_p), x_a) =\\
      f'(\omega_m, x_q, f(\omega_n, x_p, x_a)) = \\
      T'(x_q, f(\omega_n, x_p, x_a)) + G'(\omega_m).
\end{align*}
Since the random variable $f(\_, x_p, x_a): \Omega^{n} \rightarrow \rl^b$ is multivariate normal over $(\Omega^{n}, \bc(\Omega^{n}), \mu^{n})$, by the note above we have that the random variable:
\begin{gather*}
    f^r(\_, x_p, x_a): \Omega^{m+n} \rightarrow \rl^b
    \\
    f^r((\omega_m, \omega_n), x_p, x_a) = f(\omega_n, x_p, x_a)
\end{gather*}
defined over $(\Omega^{m+n}, \bc(\Omega^{m+n}), \mu^{m+n})$ is multivariate normal. Since $x_q$ is constant this implies that the following random variable is also multivariate normal:
\begin{gather*}
    T'(x_q, f^r(\_, x_p, x_a)): \Omega^{m+n} \rightarrow \rl^c.
\end{gather*}
Similarly, the random variable:
\begin{gather*}
    G'^{l}: \Omega^{m+n} \rightarrow \rl^b
    \\
    G'^{l}(\omega_m, \omega_n) = G'(\omega_m)
\end{gather*}
is also multivariate normal and independent of $T(x_q, f^r(\_, x_p, x_a))$. Therefore, we can write: 
\begin{align*}
    (f' \circ f)((\omega_m, \omega_n), (x_q, x_p), x_a) = \\
      T'(x_q, f(\omega_n, x_p, x_a)) + G'(\omega_m) = \\
      T'(x_q, f^{r}((\omega_m, \omega_n), x_p, x_a)) + G'^{l}(\omega_m, \omega_n).
\end{align*}
Since this is a sum of independent normally distributed random variables, the following random variable is also multivariate normal:
\begin{gather*}
    (f' \circ f)(\_, (x_q, x_p), x_a): \Omega^{m+n} \rightarrow \rl^c.
\end{gather*}
\end{proof}
As an aside, note that $\nm$ itself is not closed under composition. Suppose
\begin{gather*}
    f': \Omega^m \times \rl^q \times \rl^b \rightarrow \rl^c
    \\
    f: \Omega^n \times \rl^p \times \rl^a \rightarrow \rl^b
\end{gather*}
are in $\nm$ and that:
\begin{gather*}
    f'(\omega_m, x_q, x_b) = T'(x_q, x_b) + G'(\omega_m)
\end{gather*}
where $T'(x_q, x_b) = \|x_q\|_1 x_b$. Note that $T'$ is Gaussian preserving since the product of a constant and a Gaussian is Gaussian. Now if we write:
\begin{gather*}
    f(\omega_n, x_p, x_a) = T(x_p, x_a) + G(\omega_n)
\end{gather*}
we see that:
\begin{gather*}
    (f' \circ f): \Omega^{m+n} \times \rl^{q+p} \times \rl^a \rightarrow \rl^c
    \\
    (f' \circ f)((\omega_m, \omega_n), (x_q, x_p), x_a) =
    \|x_q\|_1 T(x_p, x_a) + \|x_q\|_1 G(\omega_n) + G'(\omega_m),
\end{gather*}
which we cannot express as a sum of a Gaussian-preserving transformation over $\rl^{q+p} \times \rl^a \rightarrow \rl^b$ and a multivariate normal random variable defined on $(\Omega^{n+m}, \mathcal{B}(\Omega^{n+m}), \mu^{n + m})$. $(f' \circ f)$ is therefore not in $\nm$. However, for any choice of $x_q \in \rl^q, x_p \in \rl^p, x_a \in \rl^a$ the random variable:
\begin{gather*}
    (f' \circ f)(\_, (x_q, x_p), x_a): \Omega^{n+m} \rightarrow \rl^c
\end{gather*}
is a linear function of multivariate normal random variables and is therefore itself multivariate normal.

\subsubsection{Relationship to $\gauss$}
$\dfn$ is similar to the category $\gauss$ from Section 6 of Fritz \cite{fritz2020synthetic}, with a few key differences.
\begin{definition}
In the category $\gauss$ \citep{fritz2020synthetic} objects are natural numbers and morphisms $a \rightarrow b$ are tuples $(M, C, s)$ where $M$ is a matrix in $\rl^{b \times a}$, $C$ is a positive semidefinite matrix in $\rl^{b \times b}$ and $s$ is a vector in $\rl^b$.
\end{definition}
Intuitively, the morphisms in $\gauss$ represent transformations of random variables. That is, $(M, C, s)$ implicitly represents the following transformation of random variables:
\begin{gather*}
    g(f) = Mf + \xi_{s, C}.
\end{gather*}
where $\xi_{s, C}$ is a multivariate normal random variable with mean $s$ and covariance matrix $C$ that is independent of $f$. If the random variable $f$ is normally distributed, then $g(f)$ is as well.

A primary difference between $\gauss$ and $\dfn$ is that the morphisms in $\dfn$ explicitly include the functional form of $\xi_{s, C}$ in the morphism itself. For any arrow $(M, C, s): a\rightarrow b$ in $\gauss$ and a choice of such an $\xi_{s, C}$ over $\baseprob$, we can form the $\dfn$ arrow:
\begin{gather*}
    f': \Omega \times \rl^0 \times \rl^a \rightarrow \rl^b
\end{gather*}
where for $\omega \in \Omega, x_a \in \rl^a$:
\begin{gather*}
    f'(\omega, x_a) = Mx_a + \xi_{s, C}(\omega).
\end{gather*}
However, this arrow is dependent on the choice of $\xi_{s, C}$.
% However, since this arrow is dependent on the choice of $\xi_{s, C}$, this mapping is not functorial.

\subsection{Expectation Composition}\label{section:expectation-composition}

\begin{definition}\label{ExpectationClosureDef}
A subcategory $\cb$ of $\df$ is an Expectation Composition category if for any $f: \Omega^n \times \rl^p \times \rl^a \rightarrow \rl^b$ and $f': \Omega^m \times \rl^q \times \rl^b \rightarrow \rl^c$ in $\cb$ and $x_q \in \rl^q, x_p \in \rl^p, x_a \in \rl^a$:
\begin{gather*}
    \int_{(\omega_{m}, \omega_{n}) \in \Omega^{m+n}}
    f'(\omega_m, x_q, f(\omega_{n}, x_p, x_a))
    d\mu^{m+n} = \\
    \int_{\omega_{m} \in \Omega^{m}}
    f'\left(\omega_m, x_q, 
        \int_{\omega_{n} \in \Omega^{n}} f(\omega_{n}, x_p, x_a) d\mu^n
    \right)
    d\mu^{m}.
\end{gather*}
\end{definition}

\begin{proposition}\label{ExpectationClosureProp}
$\dfn$ is an Expectation Composition category.
\end{proposition}
\begin{proof}
Consider some $f: \Omega^n \times \rl^p \times \rl^a \rightarrow \rl^b$ and $f': \Omega^m \times \rl^q \times \rl^b \rightarrow \rl^c$ in $\dfn$ and $x_q \in \rl^q, x_p \in \rl^p, x_a \in \rl^a$. We will prove by induction that Definition \ref{ExpectationClosureDef} holds.

By the definition of $\dfn$, there exists some $k\in\mathbb{N}$ such that we can express $f'$ as a composition of $k$ arrows in $\nm$. First note that if $k=1$, then $f'$ is in $\nm$, and the statement must hold since for $x_q \in \rl^q, x_p \in \rl^p, x_a \in \rl^a$:
\begin{align*}
    \int_{(\omega_{m}, \omega_{n}) \in \Omega^{m+n}} f'(\omega_m, x_q, f(\omega_{n}, x_p, x_a)) d\mu^{m+n}  = \\
    %
    % Definition of \dfn arrows
    \int_{(\omega_{m}, \omega_{n}) \in \Omega^{m+n}}
    T'(x_q, f(\omega_{n}, x_p, x_a)) + G'(\omega_m)  d\mu^{m+n}= \\
    %
    % Move integral around
    \int_{\omega_{m} \in \Omega^{m}}  
    \int_{\omega_{n} \in \Omega^{n}} 
    T'(x_q, f(\omega_{n}, x_p, x_a))  d\mu^{n} + G'(\omega_m) d\mu^{m}= \\
    %
    % Expectation preservation property of Gaussian-preserving transformation
    \int_{\omega_{m} \in \Omega^{m}}  
    T'\left(x_q, \int_{\omega_{n} \in \Omega^{n}} f(\omega_{n}, x_p, x_a)  d\mu^{n}\right) +
        G'(\omega_m) d\mu^{m}= \\
    %
    % Definition of \dfn arrows
    \int_{\omega_{m} \in \Omega^{m}}
    f'\left(\omega_m, x_q, 
        \int_{\omega_{n} \in \Omega^{n}} f(\omega_{n}, x_p, x_a) d\mu^n
    \right)
    d\mu^{m}.
\end{align*}
Next, if $k > 1$ then we can express
% NOTE: This is \df composition, with two layers of parameterization
$f' = h \circ f'_{k-1}$, where $h$ is in $\nm$ and $f'_{k-1}$ is the composition of $k-1$ arrows in $\nm$. Without loss of generality we will assume $f'_{k-1}$ and $h$ have the following signatures:
\begin{gather*}
    f'_{k-1}: \Omega^{m'} \times \rl^{q'} \times \rl^b \rightarrow \rl^{d} \qquad h: \Omega^{m''} \times \rl^{q''} \times \rl^{d} \rightarrow \rl^{c}.
\end{gather*}
Note that $q'+q'' = q$ and $m'+m'' = m$. Now we can show the following, where the step marked $*$ holds by induction and $x_{q''} \in \rl^{q''}, x_{q'} \in \rl^{q'}, x_p \in \rl^p, x_a \in \rl^a$:
\begin{align*}
    \int_{(\omega_{m''}, \omega_{m'}, \omega_{n}) \in \Omega^{m'' + m'+n}}
    f'((\omega_{m''}, \omega_{m'}), (x_{q''}, x_{q'}), f(\omega_{n}, x_p, x_a)) d\mu^{m'' + m'+n}  = \\
    %
    % Break f' into h \circ f'_{k-1}
    \int_{(\omega_{m''}, \omega_{m'}, \omega_{n}) \in \Omega^{m'' + m'+n}}
    h(\omega_{m''}, x_{q''}, f'_{k-1}(\omega_{m'}, x_{q'}, f(\omega_{n}, x_p, x_a))) d\mu^{m'' + m'+n}  = \\
    %
    % Break h into T and G
    \int_{(\omega_{m''}, \omega_{m'}, \omega_{n}) \in \Omega^{m'' + m'+n}}
    T_h(x_{q''}, f'_{k-1}(\omega_{m'}, x_{q'}, f(\omega_{n}, x_p, x_a))) + G_h(\omega_{m''})  d\mu^{m'' + m' + n}= \\
    %
    % Move the (\omega_{m'}, \omega_{n}) integral through by the definition of T
    \int_{\omega_{m''} \in \Omega^{m''}}
    T_h\left(
    x_{q''}, 
    \int_{(\omega_{m'}, \omega_{n}) \in \Omega^{m' + n}}  f'_{k-1}(\omega_{m'}, x_{q'}, f(\omega_{n}, x_p, x_a))
        d\mu^{m' + n}
    \right)
     + G_h(\omega_{m''})  d\mu^{m''} =^{*} \\
    %
    %
    % Apply the inductive step to pass the \omega_n integral through
    \int_{\omega_{m''} \in \Omega^{m''}}
    T_h\left(
    x_{q''}, 
    \int_{\omega_{m'} \in \Omega^{m'}}  
    f'_{k-1}\left(\omega_{m'}, x_{q'}, 
    \int_{\omega_{n} \in \Omega^{n}}  f(\omega_{n}, x_p, x_a)
        d\mu^n
    \right)d\mu^{m'}\right)
     + G_h(\omega_{m''})  d\mu^{m''} = \\
    %
    % Move m' integral back out by the definition of T
    \int_{(\omega_{m''}, \omega_{m'}) \in \Omega^{m'' + m'}}
    T_h\left(
    x_{q''}, 
    f'_{k-1}\left(\omega_{m'}, x_{q'}, 
    \int_{\omega_{n} \in \Omega^{n}}  f(\omega_{n}, x_p, x_a)
        d\mu^n
    \right)\right)
     + G_h(\omega_{m''})  d\mu^{m'' + m'} = \\
    %
    %
    % Recombine T_h and G_h into h
    \int_{(\omega_{m''}, \omega_{m'}) \in \Omega^{m'' + m'}}
    h\left(
    \omega_{m''}, x_{q''}, 
    f'_{k-1}
    \left(\omega_{m'}, x_{q'}, 
    \int_{\omega_{n} \in \Omega^{n}}  
    f(\omega_{n}, x_p, x_a)
    d\mu^n
    \right)\right)
    d\mu^{m'' + m'} = \\
    %
    % Recombine h and f'_{k-1} into f' (DF composition)
    \int_{(\omega_{m''}, \omega_{m'}) \in \Omega^{m'' + m'}}
    f'\left((\omega_{m''}, \omega_{m'}), (x_{q''}, x_{q'}),
        \int_{\omega_{n} \in \Omega^{n}} f(\omega_{n}, x_p, x_a) d\mu^n
    \right)
    d\mu^{m'' + m'}.
\end{align*}
By induction we have that the original statement holds for all $f',f \in \dfn$.
\end{proof}

We can now define the following functor:
\begin{proposition}\label{proposition:expectation-functor}
Suppose $\cb \subseteq \df$ is an Expectation Composition category. We can define a map $\expectation: \cb \rightarrow \para(\eucmeas)$ that acts as the identity on objects and sends the arrow $f: \Omega^n \times \rl^p \times \rl^a \rightarrow \rl^b$ in $\cb$ to the following function:
\begin{gather*}
    f_E: \rl^p \times \rl^a \rightarrow \rl^b 
    \\
    f_E(x_p, x_a) = E_{\mu^n}[f(\_,x_p, x_a)] =
    \int_{\omega_{n} \in \Omega^{n}}
    f(\omega_n,x_p, x_a) d\mu^n.
\end{gather*}
$\expectation: \cb \rightarrow \para(\eucmeas)$ is a functor.
\end{proposition}
\begin{proof}\label{proof:expectation-functor}
To start, note that $\expectation$ trivially sends objects in $\cb$ to objects in $\para(\eucmeas)$. Next, note that for any morphism in $f: \Omega^n \times \rl^p \times \rl^a \rightarrow \rl^b$ in $\cb$ the
% NOTE: The wikipedia statement of the Leibnitz rule requires that the derivative of f is bounded by an integrable function of \omega (point 3 in "Measure theory statement" section in https://en.wikipedia.org/wiki/Leibniz_integral_rule) at ANY x. This is not the case in general, but I don't think this is required for differentiability (just for the equivalence in this section)
%
% NOTE: Section 8.6.3 of https://math.byu.edu/~bakker/Math346/Lectures/M346Lec17.pdf describes the Leibniz integral rule in a way that suggests infinite differentiability of is all we need for differentiability of \int f
Leibniz integration rule implies that the following function is differentiable and
% NOTE: Continuity implies Borel measurability https://math.stackexchange.com/questions/1654046/show-that-continuous-functions-on-mathbb-r-are-borel-measurable
therefore also Borel measurable:
\begin{gather*}
    f_E: \rl^p \times \rl^a \rightarrow \rl^b 
    \\
    f_E(x_p, x_a) = E_{\mu^n}[f(\_,x_p, x_a)] =
    \int_{\omega_{n} \in \Omega^{n}}
    f(\omega_n,x_p, x_a) d\mu^n.
\end{gather*}
Therefore $\expectation$ sends morphisms in $\cb$ to morphisms in $\para(\eucmeas)$. Next, we can see that $\expectation$ preserves identities since $\expectation(id)$ is the identity function in $\para(\eucmeas)$
\begin{gather*}
    \expectation(id)(x_a) = E_{\mu^n}[id(\_, x_a)] = E_{\mu^n}[x_a] = x_a
\end{gather*}
Finally, consider the morphisms $f: \Omega^n \times \rl^p \times \rl^a \rightarrow \rl^b$ and $f': \Omega^m \times \rl^q \times \rl^b \rightarrow \rl^c$ in $\cb$. We have that for $x_q \in \rl^q, x_p \in \rl^p, x_a \in \rl^a$:
\begin{align*}
    \expectation (f' \circ f)(x_q, x_p, x_a) = \\
    %
    % Definition of \expectation
    \int_{(\omega_{m}, \omega_{n}) \in \Omega^{m+n}}
    (f' \circ f)
    ((\omega_m, \omega_n), (x_q, x_p), x_a)
    d\mu^{m+n} = \\
    %
    % Definition of \df composition
    \int_{(\omega_{m}, \omega_{n}) \in \Omega^{m+n}}
    f'(\omega_m, x_q, f(\omega_{n}, x_p, x_a))
    d\mu^{m+n} =^{*} \\
    %
    % Definition of an Expectation Composition category
    %
    \int_{\omega_{m} \in \Omega^{m}}
    f'\left(\omega_m, x_q, 
        \int_{\omega_{n} \in \Omega^{n}} f(\omega_{n}, x_p, x_a) d\mu^n
    \right)
    d\mu^{m} = \\
    %
    % Definition of \expectation
    \expectation (f')(x_q, \expectation (f)(x_p, x_a))
\end{align*}
where the step marked with $*$ is by the definition of an Expectation Composition category. This implies that $\expectation$ preserves composition.
\end{proof}

\section{Likelihood and Learning}\label{section:likelihood}
In this section we will apply the maximum likelihood procedure to the arrows in $\df$ to derive the error function $er: \rl \times \rl \rightarrow \rl$. We will then use this error function to define a modification of Fong et al.'s \cite{fong2019backprop} backpropagation functor. However, since different arrows in $\df$ have likelihood functions of different forms, we will not define a single backpropagation functor out of $\df$. Instead, we will define multiple functors from subcategories of $\df$ into $\learn$.

To do this, we will first define a substructure of $\df$ with well-defined likelihood functions. Then, we will describe a class of subcategories of $\df$ derived from this substructure. Finally, we will define a backpropagation functor for any subcategory in this class. 

\subsection{Conditional Likelihood}
The conditional likelihood is a general measure of the goodness of fit of a set of parameters and observed data for a given parametric statistical model. We can define the conditional likelihood of a parametric statistical model $f: \Omega^n \times \rl^p \times \rl^a \rightarrow \rl^b$ over the probability space $\baseprobn$ at the points $x_p \in \rl^p, x_a \in \rl^a, x_b \in \rl^b$ in terms of the pushforward measure of $\mu^n$ along the random variable $f(\_, x_p, x_a)$. To do this, we evaluate the Radon-Nikodym derivative of the probability measure:
\begin{gather*}
    f(\_, x_p, x_a)_{*}\mu^n: \bc(\rl^b) \rightarrow [0,1]\\
    f(\_, x_p, x_a)_{*}\mu^n = \mu^n(f(\_, x_p, x_a)^{-1})
\end{gather*}
with respect to a reference measure at the point $x_b$. In this work we select the Lebesgue measure over $\rl^b$, $\lambda^b$, as the reference measure. 
Note that the Radon-Nikodym derivative with respect to the Lebesgue measure is not defined for all measures. For example, no discrete measure has a Radon-Nikodym derivative with respect to the Lebesgue measure, since for any finite collection of points $A$ in $\rl^b$, $\lambda^b(A) = 0$.

Formally
% the value of the conditional likelihood function associated with the statistical model $f: \Omega^n \times \rl^p \times \rl^a \rightarrow \rl^b$ over the probability space $\baseprobn$ at $x_p, x_a, x_b$ is the value of the Radon-Nikodym derivative of the pushforward of the random variable $f(\_, x_p, x_a)$ evaluated at $x_b$ (if it exists). That is,
the conditional likelihood function for $f: \Omega^n \times \rl^p \times \rl^a \rightarrow \rl^b$ is:
\begin{gather*}
    L_f: \rl^p \times \rl^a \times \rl^b \rightarrow \rl
\end{gather*}
where for $x_p \in \rl^p, x_a \in \rl^a, x_b \in \rl^b$:
\begin{gather*}
    L_f(x_p, x_a, x_b) = \frac{d \ f(\_, x_p, x_a)_{*}\mu^n}{d\lambda^b}(x_b).
\end{gather*}
%
% Note that $\lambda^b$ is the Lebesgue measure over $\rl^b$.
For example, the conditional likelihood function for the univariate linear regression model:
\begin{gather*}
    l: \Omega^n \times \rl^3 \times \rl \rightarrow \rl
\end{gather*}
that we introduced in Section \ref{df} is:
\begin{gather*}
    L_l: \rl^3 \times \rl \times \rl \rightarrow \rl
\end{gather*}
where for $[a,b,s] \in \rl^3, x \in \rl, y \in \rl$:
\begin{gather*}
    L_l([a,b,s], x, y) = 
    \frac{1}{s\sqrt{2\pi} }\exp\left(
    -\frac{(y - (ax + b))^2}{2s^2}
    \right).
\end{gather*}
%
% An \textbf{abstract conditional likelihood} is a function that accepts a data point $(x_a,x_b)$ and a parameter value $x_p$ and returns some notion of the ``likelihood'' of observing the output value $x_b$ given the parameter value $x_p$ and the input value $x_a$. We can formally define an abstract conditional likelihood as a
\begin{definition}
An abstract conditional likelihood from $\rl^a$ to $\rl^b$ is a Borel-measurable and Lebesgue-integrable function of the form $L: \rl^p \times \rl^a \times \rl^b \rightarrow \rl$.
\end{definition}

%
% Now recall the following definition:
% %
% \begin{definition}
% A monoidal semicategory is a monoid object in $\semicat$, the monoidal category of semicategories. 
% \end{definition}
% %
% Monoidal semicategories are similar to monoidal categories but lack identity morphisms.
We can define a semicategory $\cl$ of abstract conditional likelihoods.
\begin{proposition}\label{proposition-clsemicategory}
% https://ncatlab.org/nlab/show/semicategory
% https://ncatlab.org/nlab/show/monoid+in+a+monoidal+category
We can define a semicategory $\cl$ in which objects are spaces of the form $\rl^n$ for some $n \in \mathbb{N}$ and
% The tensor of the objects $\rl^a$ and $\rl^b$ in $\cl$ is defined to be $\rl^{a+b}$. The unit of this tensor is $\rl^0$.
the morphisms between $\rl^a$ and $\rl^b$ are equivalence classes of abstract conditional likelihood functions such that for $L, L^{*}: \rl^p \times \rl^a \times \rl^b \rightarrow \rl$ we have $L \sim L^{*}$ if for all $x_p \in \rl^p, x_a \in \rl^a$, the functions 
$L(x_p, x_a, \_): \rl^b \rightarrow \rl$ and $L^{*}(x_p, x_a, \_): \rl^b \rightarrow \rl$ are
% NOTE: The Lebesgue measure is complete, so almost-everywhere equivalence is transitive and therefore forms an equivalence relation https://math.stackexchange.com/questions/276870/why-equality-almost-everywhere-is-transitive
$\lambda^b$-a.e. equivalent.

We define the composition
% and tensor
of these equivalence classes in terms of their representatives. That is, consider the equivalence classes $\mathbf{L}$ and $\mathbf{L'}$ and suppose $L_i: \rl^p \times \rl^a \times \rl^b \rightarrow \rl$ is in $\mathbf{L}$ and $L'_j: \rl^q \times \rl^b \times \rl^c \rightarrow \rl$ is in $\mathbf{L'}$. Then the representatives of $\mathbf{L'} \circ \mathbf{L}$ are:
\begin{gather*}
    (L'_j \circ L_i): \rl^{q+p}  \times \rl^a \times \rl^c  \rightarrow \rl
    \\
    (L'_j \circ L_i)((x_q, x_p), x_{a}, x_{c}) = \int_{x_b \in \rl^b} L'_j(x_q, x_b, x_c)L_i(x_p, x_a, x_b) dx_b.
\end{gather*}
for $L_i \in \mathbf{L}, L'_j \in \mathbf{L'}$. 
%
% Similarly, we can define a tensor product of abstract conditional likelihoods. The tensor of $L': \rl^q \times \rl^c \times \rl^d \rightarrow \rl$ and $L: \rl^p \times \rl^a \times \rl^b \rightarrow \rl$ is:
% %
% \begin{gather*}
%     (L' \otimes L): \rl^{q+p} \times \rl^{c+a} \times \rl^{d+b} \rightarrow \rl
% \end{gather*}
% %
% where for $x_q \in \rl^q, x_p \in \rl^p, x_c \in \rl^c, x_a \in \rl^a, x_d \in \rl^d, x_b \in \rl^b$:
% %
% \begin{gather*}
%     (L' \otimes L)((x_q, x_p), (x_c, x_a), (x_d, x_b)) = L'(x_q, x_c, x_d)L(x_p, x_a, x_b).
% \end{gather*}
% The tensor of equivalence classes is defined similarly.
\end{proposition}

\begin{proof}\label{proof-clsemicategory}
We need to show that $\cl$ is closed under composition and that composition in $\cl$ is associative. 

To start, note that if $L_i: \rl^p \times \rl^a \times \rl^b \rightarrow \rl$ and $L'_j: \rl^q \times \rl^b \times \rl^c \rightarrow \rl$ are abstract conditional likelihood functions then their composition is also an abstract conditional likelihood function since:
\begin{gather*}
    (L'_j \circ L_i): \rl^{q+p}  \times \rl^a \times \rl^c  \rightarrow \rl
    \\
    (L'_j \circ L_i)((x_q, x_p), x_{a}, x_{c}) = \int_{x_b \in \rl^b} L'_j(x_q, x_b, x_c)L_i(x_p, x_a, x_b) dx_b.
\end{gather*}
%
% NOTE: Borel measurability and Lebesgue integrability are preserved under product. Then I think (but am not 100% sure) that Fubini's theorem implies that the x_b-partial integral of the (\rl^{q+p}  \times \rl^a \times \rl^b \times \rl^c)-integrable function is both Borel-measurable and (\rl^{q+p}  \times \rl^a \times \rl^c)-integrable
is also Borel-measurable and Lebesgue integrable.

Next, we need to show that for any pair of equivalence classes:
\begin{gather*}
    \mathbf{L}: \rl^a \rightarrow \rl^b
    \qquad
    \mathbf{L'}: \rl^b \rightarrow \rl^c
\end{gather*}
and choice of representatives:
\begin{gather*}
    L_1: \rl^p \times \rl^a \times \rl^b \rightarrow \rl
    \qquad
    L_2: \rl^p \times \rl^a \times \rl^b \rightarrow \rl
\end{gather*}
in $\mathbf{L}$ and:
\begin{gather*}
    L'_1: \rl^q \times \rl^b \times \rl^c \rightarrow \rl
    \qquad
    L'_2: \rl^q \times \rl^b \times \rl^c \rightarrow \rl
\end{gather*}
in $\mathbf{L'}$ we have that for any $x_q \in \rl^q, x_p \in \rl^p, x_a \in \rl^a$ the functions:
\begin{align*}
    &(L'_1 \circ L_1)((x_q, x_p), x_a, \_): \rl^c \rightarrow \rl
    \\
    &(L'_2 \circ L_2)((x_q, x_p), x_a, \_): \rl^c \rightarrow \rl
\end{align*}
are $\lambda^c$-a.e. equivalent.

Define $\sigma_c$ to be the set of all $x_c \in \rl^c$ where:
\begin{gather*}
    (L'_1 \circ L_1)((x_q, x_p), x_a, x_c) \neq (L'_2 \circ L_2)((x_q, x_p), x_a, x_c)
\end{gather*}
We need to show that $\sigma_c$ has Lebesgue measure $0$. For any $x_c \in \rl^c$, define $\sigma_b(x_c)$ to be the union of the following subsets of $\rl^b$:
\begin{gather*}
    % These two subsets correspond to the regions where the first or second composition do not match
    \sigma_b(x_c) = 
    \{x_b \ | \ L_1(x_p, x_a, x_b) \neq L_2(x_p, x_a, x_b)\}
    \cup 
    \{x_b \ | \ L'_1(x_q, x_b, x_c) \neq L'_2(x_q, x_b, x_c)\}
\end{gather*}
Now for any $x_c$ where the Lebesgue measure of $\sigma_b(x_c)$ is $0$ we have the following:
% 
% NOTE: We need to show that this expression can be written without the bi
\begin{align*}
    (L'_1 \circ L_1)((x_q, x_p), x_a, x_c) = \\
    \int_{x_b \in \rl^b} L'_1(x_q, x_b, x_c)L_1(x_p, x_a, x_b) dx_b
    = \\
    %
    % Expand out the integral
    \int_{x_b \in \rl^b - \sigma_b(x_c)} L'_1(x_q, x_b, x_c)L_1(x_p, x_a, x_b) dx_b
    + 
    \int_{x_b \in \sigma_b(x_c)} L'_1(x_q, x_b, x_c)L_1(x_p, x_a, x_b) dx_b = \\
    %
    % \sigma_b is a union of two regions with measure 0, and the integral of any function over a region with measure 0 is 0 https://www.physicsforums.com/threads/integral-over-a-set-of-measure-0.430720/
    \int_{x_b \in \rl^b - \sigma_b(x_c)} L'_1(x_q, x_b, x_c)L_1(x_p, x_a, x_b) dx_b = \\
    %
    % Definition of \sigma_b
    % TODO: Can we apply this step? Don't we need to also enforce something on x_c? 
    %
    \int_{x_b \in \rl^b - \sigma_b(x_c)} L'_2(x_q, x_b, x_c)L_2(x_p, x_a, x_b) dx_b = \\
    %
    % \sigma_b has measure 0
    \int_{x_b \in \rl^b - \sigma_b(x_c)} L'_2(x_q, x_b, x_c)L_2(x_p, x_a, x_b) dx_b
    + 
    \int_{x_b \in \sigma_b(x_c)} L'_2(x_q, x_b, x_c)L_2(x_p, x_a, x_b) dx_b = \\
    %
    % Rewrite the integral
    \int_{x_b \in \rl^b} L'_2(x_q, x_b, x_c)L_2(x_p, x_a, x_b) dx_b = \\
    (L'_2 \circ L_2)((x_q, x_p), x_a, x_c)
\end{align*}
Therefore, for any $x_c \in \sigma_c$ it must be that the Lebesgue measure of $\sigma_b(x_c)$ is greater than $0$. 

Since $L_1, L_2$ are representatives of the same equivalence class it must be that the set:
\begin{gather*}
    \{x_b \ | \ L_1(x_p, x_a, x_b) \neq L_2(x_p, x_a, x_b)\} \subseteq \rl^b
\end{gather*}
always has Lebesgue measure $0$, and therefore $\sigma_c$ is equal to the set of all $x_c$ for which the set:
\begin{gather*}
    \{x_b \ | \ L'_1(x_q, x_b, x_c) \neq L'_2(x_q, x_b, x_c)\} \subseteq \rl^b
\end{gather*}
has Lebesgue measure greater than $0$.

Now suppose for contradiction that the set $\sigma_c$ has Lebesgue measure greater than $0$. Then the set:
% NOTE: This is a bit of a jump and I'm not totally sure this is right. But a rough proof of this is right here: https://math.stackexchange.com/questions/3004809/almost-everywhere-zero-in-product-measure-space
\begin{gather*}
    \{(x_b, x_c) \ | \ L'_1(x_q, x_b, x_c) \neq L'_2(x_q, x_b, x_c)\} \subseteq \rl^{b+c}
\end{gather*}
must have Lebesgue measure greater than $0$ as well. However, this is impossible
% by Fubini's theorem
since $L'_1, L'_2$ are representatives of the same equivalence class and therefore for any fixed $x_b \in \rl^b$ the set:
\begin{gather*}
    \{x_c \ | \ L'_1(x_q, x_b, x_c) \neq L'_2(x_q, x_b, x_c)\} \subseteq \rl^c
\end{gather*}
must have Lebesgue measure equal to $0$. Therefore  $\sigma_c$ has Lebesgue measure equal to $0$.

%
%
% NOTE: Do not delete!!!!!!!!
%
%
%
% Therefore we simply need to show that the Lebesgue measure of $\sigma_c$ is $0$. We can write:
% \begin{align*}
% \lambda^c(\sigma_c) =  \\
% %
% % Definition of \lambda_c
% \int_{x_c \in \rl^c}
% \delta(x_c, \sigma_c)
% dx_c = \\
% %
% % Definition of \sigma_c
% \int_{x_c \in \rl^c}
% \delta(x_c, \{x_c \ | \ 
% \left[
%     \int_{x_b \in \rl^b}
%     \delta(x_b, \{x_b \ | \ L'_1(x_q, x_b, x_c) \neq L'_2(x_q, x_b, x_c))
%     dx_b
% \right] > 0
% \})
% dx_c \leq \\
% %
% % ?????
% \int_{x_c \in \rl^c}
% \int_{x_b \in \rl^b} 
% \delta(x_c, \{x_c \ | \
%   L'_1(x_q, x_b, x_c) \neq L'_2(x_q, x_b, x_c)
% \})
% dx_b
% dx_c= \\
% %
% % Rearrange
% \int_{x_b \in \rl^b} 
% \int_{x_c \in \rl^c}
% \delta(x_c, \{x_c \ | \
%   L'_1(x_q, x_b, x_c) \neq L'_2(x_q, x_b, x_c)
% \})
% dx_c
% dx_b= \\
% %
% % Since L'_1 and L'_2 are in the same equivalence class, for every choice of (x_q, x_b) the set of x_c where the inequality holds has measure 0
% \int_{x_b \in \rl^b} 
% 0
% dx_b = \\
% 0
% \end{align*}
%
Therefore since $\sigma_c$ has Lebesgue measure $0$ we can conclude that:
\begin{align*}
    &(L'_1 \circ L_1)((x_q, x_p), x_a, \_): \rl^c \rightarrow \rl
    \\
    &(L'_2 \circ L_2)((x_q, x_p), x_a, \_): \rl^c \rightarrow \rl
\end{align*}
are $\lambda^c$-a.e. equivalent and $\cl$ is closed under composition.

Next,
% since we have already shown that $\cl$ is closed under composition
we need to show that composition is associative. Suppose the following are representatives of three arrows in $\cl$:
\begin{gather*}
    f_1: \rl^{p_1} \times \rl^a \times \rl^b \rightarrow \rl
    \\
    f_2: \rl^{p_2} \times \rl^b \times \rl^c \rightarrow \rl
    \\
    f_3: \rl^{p_3} \times \rl^c \times \rl^d \rightarrow \rl
\end{gather*}
Now consider the representatives of their composition:
\begin{align*}
    f_3\circ (f_2 \circ f_1): \rl^a \rightarrow \rl^d
    \\
    (f_3\circ f_2) \circ f_1: \rl^a \rightarrow \rl^d
\end{align*}
For $x_{p_3} \in \rl^{p_3}, x_{p_2} \in \rl^{p_2}, x_{p_1} \in \rl^{p_1}, x_a \in \rl^a, x_d \in \rl^d$ we then have:
\begin{align*}
    (f_3\circ (f_2 \circ f_1))((x_{p_3}, x_{p_2}, x_{p_1}), x_a, x_d) = \\
    \int_{x_c \in \rl^c}
    f_3(x_{p_3}, x_c, x_d)
    \left(\int_{x_b \in \rl^b}
    f_2(x_{p_2}, x_b, x_c)
    f_1(x_{p_1}, x_a, x_b) dx_b\right) dx_c = \\
    %Fubini
    %
    \int_{x_c \in \rl^c}
    \int_{x_b \in \rl^b}
    f_3(x_{p_3}, x_c, x_d)
    f_2(x_{p_2}, x_b, x_c)
    f_1(x_{p_1}, x_a, x_b)  dx_b dx_c = \\
    %Fubini
    %
    \int_{x_b \in \rl^b}
    \left(\int_{x_c \in \rl^c}
    f_3(x_{p_3}, x_c, x_d)
    f_2(x_{p_2}, x_b, x_c) 
    dx_c\right)
    f_1(x_{p_1}, x_a, x_b) dx_b   = \\
    ((f_3\circ f_2) \circ f_1) ((x_{p_3}, x_{p_2}, x_{p_1}), x_a, x_d)
\end{align*}
Therefore, composition in $\cl$ is associative, so $\cl$ is a semicategory.
\end{proof}
Note that $\cl$ does not form a category because objects in $\cl$ do not necessarily have identities. For example, for $b>0$ there is no function $\delta_b: \rl^0 \times \rl^b \times \rl^b \rightarrow \rl$ such that the following holds for all $L: \rl^p \times \rl^a \times \rl^b \rightarrow \rl$ and $x_p \in \rl^p, x_a \in \rl^a, x_b \in \rl^b$:
\begin{gather*}
    (\delta_b \circ L): \rl^p \times \rl^a \times \rl^b \rightarrow \rl
    \\
    (\delta_b \circ L)(x_{p}, x_{a}, x_{b}) = 
    \int_{x'_b \in \rl^b} \delta_b(x_b, x'_b)L(x_p, x_a, x'_b) dx'_b =
    L(x_p, x_a, x_b).
\end{gather*}

If we extend from functions to generalized functions (distributions) we can form a category similar to $\cl$. For example, Blute \cite{blute2007conformal} defines a category $\mathbf{DRel}$ of tame distributions in which the Dirac delta $\delta$ exists as a singular distribution. The semicategory $\cl$ is similar in spirit to the nuclear ideal of $\mathbf{DRel}$ that Blute et al. describe. However, we will use conditional likelihood functions to define optimization objectives, and there is no obvious way to do this with a singular distribution. For this reason we will keep $\cl$ as a semicategory.

Next, given a probability space $\baseprob$ define $\dfr$ to be the substructure of $\df$ with the same objects, but with morphisms between $\rl^a$ and $\rl^b$ limited to $f: \Omega^n \times \rl^p \times \rl^a \rightarrow \rl^b$ such that the following Borel-measurable and Lebesgue-integrable function exists:
\begin{gather*}
    L_f: \rl^p \times \rl^a \times \rl^b \rightarrow \rl
    \\
    L_f(x_p, x_a, x_b) = \frac{d \ f(\_, x_p, x_a)_{*}\mu^n}{d\lambda^b}(x_b)
\end{gather*}
That is, we have:
\begin{gather*}
    f(\_, x_p, x_a)_{*}\mu^{n}(\sigma_b) =
    \int_{x_b \in \sigma_b} L_f(x_p, x_a, x_b) 
    d\lambda^b
\end{gather*}
\begin{proposition}\label{proposition-RadonNikodymClosed}
$\dfr$ is a semicategory.
\end{proposition}
\begin{proof}\label{proof-RadonNikodymClosed}
% (NOT NECESSARY AFTER WE DITCHED MONOID STRUCTURE) Does proving that $\dfr$ is a sub-semicategory of a monoidal category and is closed under tensor automatically prove that it is a monoid object in SemiCat?
Since composition in $\dfr$ is the same as in $\df$ we simply need to show that $\dfr$ is closed under $\df$-composition.

Suppose $f: \Omega^n \times \rl^p \times \rl^a \rightarrow \rl^b$ and $f': \Omega^m \times \rl^q \times \rl^b \rightarrow \rl^c$ are arrows in $\dfr$. We can show that for all $x_q \in \rl^q, x_p \in \rl^p, x_a \in \rl^a$ there exists some Borel-measurable and Lebesgue integrable $g: \rl^c \rightarrow \rl$ such that for $\sigma_{c} \in \bc(\rl^c)$:
\begin{gather*}
    (f' \circ f)(\_, (x_q, x_p), x_a)_{*}\mu^{m+n}: \bc(\rl^c) \rightarrow [0,1]
    \\
    (f' \circ f)(\_, (x_q, x_p), x_a)_{*}\mu^{m+n}(\sigma_c) =
    \int_{x_c \in \sigma_c} g(x_c) d\lambda^c
\end{gather*}
where $\lambda^c$ is the Lebesgue measure over $\rl^c$:
\begin{align*}
    (f' \circ f)(\_, (x_q, x_p), x_a)_{*}\mu^{m+n}(\sigma_c) = \\
    %
    %
    % Definition of pushforward
    \int_{(\omega_m,\omega_n) \in \Omega^m \times \Omega^n}
    \delta((f' \circ f)((\omega_m, \omega_n), (x_q, x_p), x_a), \sigma_c) d\mu^{n+m} = \\
    % 
    % Separate out the \omega_m and \omega_m and apply \df composition
    \int_{\omega_m \in \Omega^m}
    \int_{\omega_n \in \Omega^n}
    \delta(f'(\omega_m, x_q, f(\omega_n, x_p, x_a)), \sigma_c) d\mu^n d\mu^m = \\
    %
    % Apply identity in https://en.wikipedia.org/wiki/Dirac_measure that:
    %   f(y) = \int_{x \in X} f(x) d\delta(y, \_). Then apply Fubini to rearrange integral
    %
    \int_{x_b \in \rl^b}
    \int_{\omega_m \in \Omega^m}
    \int_{\omega_n \in \Omega^n}
    \delta(f'(\omega_m, x_q, x_b), \sigma_c) d\delta(f(\omega_n, x_p, x_a), \_) d\mu^n d\mu^m  = \\
    %
    % Pass \omega_n through the integral and apply https://mathoverflow.net/questions/412896/does-the-radon-nikodym-derivative-commute-with-integration
    \int_{x_b \in \rl^b}
    \left[
    \int_{\omega_m \in \Omega^m}
    \delta(f'(\omega_m, x_q, x_b), \sigma_c)
    d\mu^m
    \right]
    d\left[
    \int_{\omega_n \in \Omega^n}
    \delta(f(\omega_n, x_p, x_a), \_)
    d\mu^n
    \right] =
    \\
    %
    % Definition of Pushforward
    \int_{x_b \in \rl^b}
    f'(\_, x_q, x_b)_{*}\mu^m(\sigma_c)\ 
    df(\_, x_p, x_a)_{*}\mu^n = \\
    %
    % Rewriting f' and f in terms of their radon-nikodym derivatives. Note that we drop the d<...> now since we are just taking all integrals w.r.t Lebesgue
    \int_{x_b \in \rl^b}
    \left[\int_{x_c \in \sigma_c} \frac{d f'(\_, x_q, x_b)_{*}\mu^m}{d\lambda^c}(x_c) d\lambda^c\right]
    \left[\frac{d f(\_, x_p, x_a)_{*}\mu^n}{d\lambda^b}(x_b) d\lambda^b \right]
    = \\
    % 
    % Apply Fubini
    \int_{x_c \in \sigma_c} 
    \left[
        \left(
        \int_{x_b \in \rl^b}
        \frac{d f'(\_, x_q, x_b)_{*}\mu^m}{d\lambda^c}(x_c)
        \right)
        \left(
        \frac{d f(\_, x_p, x_a)_{*}\mu^n}{d\lambda^b}(x_b)
        d\lambda^b
        \right)
    \right]
     d\lambda^c.
\end{align*}
Therefore we have that:
\begin{align*}
    L_{f' \circ f}((x_q, x_p), x_a, x_c) = \\
    \frac{d(f' \circ f)(\_, (x_q, x_p), x_a)_{*}\mu^{m+n}}{\lambda^c}(x_c) = \\
    \int_{x_b \in \rl^b}
    \left(
    \frac{d f'(\_, x_q, x_b)_{*}\mu^m}{d\lambda^c}(x_c)
    \right)
    \left(
    \frac{d f(\_, x_p, x_a)_{*}\mu^n}{d\lambda^b}(x_b)
    \right)
    d\lambda^b = \\
    \int_{x_b \in \rl^b}
    L_{f'}(x_q, x_b, x_c)
    L_f(x_p, x_a, x_b)
    d\lambda^b 
\end{align*}
%
% NOTE: Borel measurability and Lebesgue integrability are preserved under product. Then I think (but am not 100% sure) that Fubini's theorem implies that the x_b-partial integral of the (\rl^{q+p}  \times \rl^a \times \rl^b \times \rl^c)-integrable function is both Borel-measurable and (\rl^{q+p}  \times \rl^a \times \rl^c)-integrable
and $L_{f' \circ f}$ is therefore Lebesgue integrable and Borel measurable since $L_{f'}$ and $L_f$ are Lebesgue integrable and Borel measurable.

\end{proof}

% \begin{definition}\label{monoidalsemifunctor}
% % https://ncatlab.org/nlab/show/semicategory
% % https://ncatlab.org/nlab/show/monoid+in+a+monoidal+category
% A strict monoidal semifunctor is a semifunctor $F: \cb \rightarrow \db$ such that:
% \begin{itemize}
%     \item $F(1_{\cb}) = 1_{\db}$
%     %
%     \item For $o_1,o_2 \in Ob(\cb)$, $F(o_1 \otimes o_2) = F(o_1) \otimes F(o_2)$
%     %
%     \item For $a_1, a_2 \in Ar(\cb)$, $F(a_1 \otimes a_2) = F(a_1) \otimes F(a_2)$
% \end{itemize}
% \end{definition}

\begin{proposition}\label{RadonNikodymFunctor}
% https://ncatlab.org/nlab/show/semicategory

We can define a semifunctor $\rn: \dfr \rightarrow \cl$ that acts as the identity on objects and sends any morphism $f: \Omega^n \times \rl^p \times \rl^a \rightarrow \rl^b$ in $\dfr$ to the equivalence class that contains the function:
\begin{gather*}
    L_f: \rl^p \times \rl^a \times \rl^b \rightarrow \rl
    \\
    L_f(x_p, x_a, x_b) = \frac{d f(\_, x_p, x_a)_{*}\mu^n}{d\lambda^b}(x_b).
\end{gather*}
\end{proposition}
\begin{proof}
To start, note that $\rn$ maps objects in $\dfr$ to objects in $\cl$ by definition. Next, Proposition \ref{proposition-RadonNikodymClosed} implies that for each morphism $f \in \dfr$ the function
\begin{gather*}
    L_f: \rl^p \times \rl^a \times \rl^b \rightarrow \rl
    \\
    L_f(x_p, x_a, x_b) = \frac{d f(\_, x_p, x_a)_{*}\mu^n}{d\lambda^b}(x_b).
\end{gather*}
exists and therefore $\rn$ maps morphisms in $\dfr$ to morphisms in $\cl$.

Now we will show that $\rn$ preserves composition. Suppose
\begin{align*}
    & f: \Omega^n \times \rl^p \times \rl^a \rightarrow \rl^b
    \\
    & f': \Omega^m \times \rl^q \times \rl^b \rightarrow \rl^c
\end{align*}
are arrows in $\dfr$. Then for any:
\begin{align*}
    & x_q \in \rl^q \qquad x_p \in \rl^p \\
    & x_a \in \rl^a \qquad x_c \in \rl^c
\end{align*}
we have:
\begin{gather*}
    \rn (f' \circ f): \rl^{q+p} \times \rl^a \times \rl^c \rightarrow \rl
\end{gather*}
\begin{align*}
    \rn (f' \circ f) ((x_q, x_p), x_a, x_c)) = \\
    %
    % Definition of \rn
    \frac{d
        (f' \circ f)(\_, (x_q, x_p), x_a)_{*}\mu^{m+n}
    }{d\lambda^c}(x_c) = \\
    %
    % Re-applying same steps as above to show that the pushforward of \df composition is the stoch-composition of the pushforwards
    \frac{d \int_{x_b \in \rl^b}
        f'(\_, x_q, x_b)_{*}\mu^m((\_)_c)\ \ 
        df(\_, x_p, x_a)_{*}\mu^n
    }{d\lambda^c}(x_c) = \\
    %
    % Rewriting f' in terms of its radon-nikodym derivative 
    \frac{d \int_{x_b \in \rl^b}
        \left[\int_{x'_c \in ((\_)_c)} \frac{d f'(\_, x_q, x_b)_{*}\mu^m}{d\lambda^c}(x'_c) d\lambda^c \right]
        df(\_, x_p, x_a)_{*}\mu^n
    }{d\lambda^c}(x_c) = \\
    %
    % Integration of f with respect to d\mu is the same as integration of f*r with respect to lebesgue measure where r is the radon-nikodym derivative of \mu
    \frac{d \int_{x_b \in \rl^b}
        \left[\int_{x'_c \in ((\_)_c)} \frac{d f'(\_, x_q, x_b)_{*}\mu^m}{d\lambda^c}(x'_c) d\lambda^c \right]
        \left[
        \frac{d f(\_, x_p, x_a)_{*}\mu^n)}{d\lambda^b}(x_b) d\lambda^b
        \right]
    }{d\lambda^c}(x_c) = \\
    % %
    % % Fubini
    \frac{d \int_{x'_c \in (\_)_c}  \left[
        \int_{x_b \in \rl^b}
        \frac{d f'(\_, x_q, x_b)_{*}\mu^m}{d\lambda^c}(x'_c) \ \ 
        \frac{d f(\_, x_p, x_a)_{*}\mu^n)}{d\lambda^b}(x_b) d\lambda^b
        \right] d\lambda^c
    }{d\lambda^c}(x_c) = \\
    % %
    % % Definition of Radon-Nikodym derivative
    \int_{x_b \in \rl^b}
    \frac{d f'(\_, x_q, x_b)_{*}\mu^m}{d\lambda^c}(x_c)\ \ 
    \frac{d f(\_, x_p, x_a)_{*}\mu^n)}{d\lambda^b}(x_b) d\lambda^b = \\
    %
    % Definition of \rn
    (\rn f' \circ \rn f)((x_q, x_p), x_a, x_c).
\end{align*}
\end{proof}

\subsection{Maximum Likelihood}
% NOTE: \tau is NOT a probability measure for each x_a because \tau(x_a, \sigma_b) is always 0. If we want to represent conditioning on x_a we need to do it in another way
Suppose we have a probability space $(\rl^{a} \times \rl^{b}, \bc(\rl^{a} \times \rl^{b}), \tau)$.
%
% Suppose that we also have an arrow $f: \Omega^n \times \rl^p \times \rl^a \rightarrow \rl^b$ in $\dfr$ and we want to find the $x_p \in \rl^p$ such that for each $\sigma_a \in \bc(\rl^{a})$, the distribution:
% %
% % NOTE: There are some \sigma_a where the following distributions are both undefined because the denominator is zero. 
% %
% \begin{align*}
%     \frac
%     {\int_{x_a \in \sigma_a}
%     f(\_,x_p, x_{a})_{*}\mu^n\ dx_a}
%     {\int_{x_a \in \sigma_a}
%     f(\_,x_p, x_{a})_{*}\mu^n(\rl^b)\ dx_a}: \bc(\rl^b) \rightarrow [0,1]
% \end{align*}
% %
% best approximates the distribution:
% %
% \begin{align*}
%     \frac
%     {\tau(\sigma_a, \_)}
%     {\tau(\sigma_a, \rl^b)}: \bc(\rl^b) \rightarrow [0,1]
% \end{align*}
%
The maximum expected log-likelihood estimator for $f$ with respect to $\tau$ is the vector $x_p \in \rl^p$ that maximizes the following function (note that $log$ is a monotonic transformation and we just use it to make the math easier - the optimal value of $x_p$ is the same with or without it):
\begin{gather*}
    L_{\tau}: \rl^p \rightarrow \rl
    \\
    L_{\tau}(x_p) = \int_{(x_a, x_b)\in \rl^{a} \times \rl^{b}} 
    log 
    \frac{d f(\_,x_p, x_{a})_{*}\mu^n}{d\lambda^b}(x_{b})
    d\tau.
\end{gather*}
That is, the maximum expected log-likelihood estimator for $f$ with respect to $\tau$ is the vector $x_p$ that maximizes the expected value of:
\begin{align*}
 log \frac{d f(\_,x_p, x_{a})_{*}\mu^n}{d\lambda^b}(x_{b})
\end{align*}
over $\tau$.
% Equivalently, $x_p$ minimizes the weighted sum over $x_a$ of the KL-divergences between $f(\_,x_p, x_{a})_{*}\mu^n$ and $\tau(x_a, \_)$, where the weight of each $x_a$ is determined by $\tau$ \citep{Murphy2012}. 

Now suppose that instead of observing a probability space $(\rl^{a} \times \rl^{b}, \bc(\rl^{a} \times \rl^{b}), \tau)$ directly we have a dataset of samples:
\begin{align*}
    S_n = \{(x_{a_1}, x_{b_1}), (x_{a_2}, x_{b_2}), \cdots, (x_{a_n}, x_{b_n})\}
\end{align*}
in $\rl^a \times \rl^b$.
\begin{definition}
The maximum log likelihood estimator for $f$ with respect to the samples:
\begin{gather*}
    (x_{a_1}, x_{b_1}), (x_{a_2}, x_{b_2}), \cdots, (x_{a_n}, x_{b_n}) \in \rl^a \times \rl^b
\end{gather*}
is the vector $x_p \in \rl^p$ that maximizes the function:
\begin{gather*}
    L_{S_n}(x_p): \rl^p \rightarrow \rl
    \\
    L_{S_n}(x_p) = \sum_{i=1}^n log \frac{d f(\_,x_p, x_{a_i})_{*}\mu^n}{d\lambda^b}(x_{b_i}).
\end{gather*}
\end{definition}
%
% By Chebyshev's inequality we have convergence in probability as long as E[L_{S_n}(x_p) - L_{\tau}(x_p)] -> 0 as n -> infinity. By Proposition 2 in https://engineering.purdue.edu/ChanGroup/ECE645Notes/StudentLecture08.pdf this expectation convergence holds for any distribution
Note that if we assume the samples in $S_n$ are drawn from $(\rl^{a} \times \rl^{b}, \bc(\rl^{a} \times \rl^{b}), \tau)$, then by the weak law of large numbers
$\frac{1}{n} L_{S_n}$ converges to $L_{\tau}$ in probability as $n \rightarrow \infty$.

However, it will be challenging to derive an objective function for Fong et al's \cite{fong2019backprop} backpropagation functor from $L_{S_n}$ directly, since their construction assumes that the error function has the signature $er: \rl \times \rl \rightarrow \rl$ and has an invertible derivative. We will slightly modify $L_{S_n}$ to make this easier.

For any $j \leq b$, the $j$th component of $f: \Omega^n \times \rl^p \times \rl^a \rightarrow \rl^b$ is the function:
\begin{align*}
f[j]: \Omega^n \times \rl^p \times \rl^a \rightarrow \rl
\end{align*}
and the marginal likelihood at $x_p \in \rl^p$ of this component for some sample $(x_{a_i}, x_{b_i}) \in S_n$ is:
\begin{gather*}
    l_{ij}: \rl^p \rightarrow \rl
    \\
    l_{ij}(x_p) = \frac{d f(\_,x_p, x_{a_i})[j]_{*}\mu^n}{d\lambda}(x_{b_i}[j])
\end{gather*}
where we write $x_{b_i}[j] \in \rl$ for the $j$th component of the vector $x_{b_i} \in \rl^b$. We can now define the following:

\begin{definition}
The maximum log-marginal likelihood estimator for $f$ with respect to the samples:
\begin{gather*}
    (x_{a_1}, x_{b_1}), (x_{a_2}, x_{b_2}), \cdots, (x_{a_n}, x_{b_n}) \in \rl^a \times \rl^b
\end{gather*}
is the vector $x_p \in \rl^p$ that maximizes the function:
\begin{gather*}
    M_{S_n}: \rl^p \rightarrow \rl
    \\
    M_{S_n}(x_p) = \sum_{i=1}^n\sum_{j=1}^b log\  l_{ij}(x_p).
\end{gather*}
where $l_{ij}(x_p)$ is the marginal likelihood at $x_p \in \rl^p$ of the $j$th component of $f$ for $(x_{a_i}, x_{b_i}) \in S_n$.
\end{definition}
Note that $M_{S_n}(x_p) = L_{S_n}(x_p)$ when the real-valued random variables:
\begin{align*}
    f(\_,x_p, x_{a_i})[j]: \Omega^n \rightarrow \rl
\end{align*}
are mutually independent for all $x_{a_i}$ and $j \leq b$.

This suggests a criterion for an error function $er: \rl \times \rl \rightarrow \rl$ over which we can define Fong et al.'s \cite{fong2019backprop} backpropagation functor: we want the following two real-valued functions of $\rl^p$ to move in tandem for any fixed $(x_a,y) \in \rl^a \times \rl$ and $j \leq b$:
\begin{align*}
    & l(x_p) =  er\left(E_{\mu^n}[f(\_,x_p, x_a)[j]], y\right)
    \\
    & l'(x_p) = -\frac{d f(\_,x_p, x_a)[j]_{*}\mu^n}{d\lambda}(y).
\end{align*}
We will now make this formal. 

\subsection{Learning from Likelihoods}
% IDEA: We want to define $\dfnr$ to be the limit? of $U(\dfn)$ and $\dfr$ in the restriction of $\mathbf{SemiCat}$

% Kullbrick-Leibler and Likelihood https://wiseodd.github.io/techblog/2017/01/26/kl-mle/
% KL-divergence is defined between distributions. The KL divergence between distributions can have a similar form as the likelihood function of a distribution, except the expected value of the "true" distribution will play a similar role to the input value to the likelihood function

Suppose we have a real-valued random variable $f$ over the probability space $\baseprobn$. Write $E_{\mu^n}[f] \in \rl$ for the expectation of $f$ over $\mu^n$:
\begin{gather*}
    E_{\mu^n}[f] = \int_{\omega_n \in \Omega^n} f(\omega_n) \ d\mu^n.
\end{gather*}
And define $f^{0}$ to be:
\begin{gather*}
    f^{0}(\omega_n) = f(\omega_n) - E_{\mu^n}[f].
\end{gather*}
Next,
% define $\mathbf{Cat}^{*}$ and $\mathbf{SemiCat}^{*}$ to be the subcategories of $\mathbf{Cat}^{*}$ and $\mathbf{SemiCat}$ respectively that contain only faithful, identity-on-object functors/semi-functors as morphisms.
suppose $U: \mathbf{Cat} \rightarrow \mathbf{SemiCat}$ is the forgetful functor.

\begin{definition}
An Expectation Composition category $\cb$ is a Marginal Likelihood Factorization Category over the measure $\mu: \bc(\Omega) \rightarrow [0,1]$ if
the following cospan in $\mathbf{SemiCat}$:
\begin{gather*}
U(\cb)
\xhookrightarrow{inc}
U(\df)
\xhookleftarrow{inc'}
\dfr
\end{gather*}
(where $inc$ and $inc'$ are respectively the inclusion maps of $U(\cb)$ and $\dfr$ into $U(\df)$) has a pullback
\begin{align*}
    U(\cb) \xleftarrow{h_l} \cbr \xrightarrow{h_r} \dfr
\end{align*}
that satisfies the following property. There exists:
\begin{itemize}
    \item A differentiable function with invertible derivative $er: \rl \times \rl \rightarrow \rl$ 
    \item For each $n \in \mathbb{N}$, a function $\alpha_n: (\Omega^n \rightarrow \rl) \rightarrow \rl$
    \item For each $n \in \mathbb{N}$, a non-negative function $\beta_n: (\Omega^n \rightarrow \rl) \rightarrow \rl$
\end{itemize}
such that for any
\begin{align*}
    x_p\in\rl^p,x_a\in\rl^a,j\leq b
\end{align*}
and arrow in the semicategory $\cbr$ whose image under:
\begin{gather*}
inc \circ h_l: \cbr \rightarrow U(\df)
\end{gather*}
has the signature $f: \Omega^n \times \rl^p \times \rl^a \rightarrow \rl^b$, we can write:
% where the pushforward measure $f(\_,x_p, x_{a})[j]_{*}\mu: \bc(\rl) \rightarrow [0,1]$ has a Radon-Nikodym derivative with respect to the Lebesgue measure $\lambda$ over $\rl$, then
%
\begin{gather*}
    log \ \frac{d f(\_,x_p, x_{a})[j]_{*}\mu^n}{d\lambda}: \rl \rightarrow \rl
\end{gather*}
\begin{align*}
    log \frac{d f(\_,x_p, x_{a})[j]_{*}\mu^n}{d\lambda}(y) = \\
    % 
    % The ith component of the expected value vector is the expected value of the ith marginal: https://stats.stackexchange.com/questions/185729/expected-value-of-a-marginal-distribution-when-the-joint-distribution-is-given
    \alpha_{n}(f^{0}(\_,x_p, x_a)[j]) - \beta_{n}(f^{0}(\_,x_p, x_a)[j])er\left(
    E_{\mu^n}[f(\_,x_p, x_a)[j]], 
    y\right).
\end{align*}
We will refer to $er$ as a marginal error function of $\cb$.

\end{definition}

\begin{proposition}\label{NormalMarginalLikelihoodPreserve}
$\dfn$ is a Marginal Likelihood Factorization Category with a marginal error function $er(a,b) = (a-b)^2$.
\end{proposition}
\begin{proof}
% NOTE: the Expectation Composition condition is satisfied by Proposition \ref{ExpectationClosureProp}. 
To begin, consider the structure $\cbr$ that has the same objects as $\dfr$ and:
\begin{gather*}
    \cbr[\rl^a, \rl^b] = \dfn[\rl^a, \rl^b] \cap \dfr[\rl^a, \rl^b].
\end{gather*}
Since $\dfn$ and $\dfr$ are small this intersection is well-defined. Note also that if we have:
\begin{gather*}
    f_1 \in \dfn[\rl^a, \rl^b], f_1 \in \dfr[\rl^a, \rl^b] \\
    f_2 \in \dfn[\rl^b, \rl^c], f_2 \in \dfr[\rl^b, \rl^c]
\end{gather*}
Then since $\dfn$ and $\dfr$ are closed under composition it must be that:
\begin{gather*}
    f_2 \circ f_1 \in \dfn[\rl^a, \rl^c] \cap \dfr[\rl^a, \rl^c]
\end{gather*}
Therefore $\cbr$ is a semicategory.

% Furthermore, since $\cbr$ is closed under $\df$-composition and $\df$-tensor, it is a monoidal semicategory.
Now note that there exist identity-on-objects and identity-on-morphisms inclusion semifunctors:
\begin{align*}
    & id_l: \cbr \hookrightarrow U(\dfn)
    \\
    & id_r: \cbr \hookrightarrow \dfr
\end{align*}
such that the following diagram commutes:
\begin{center}
\begin{tikzcd}[column sep=2in,row sep=2in]
\cbr \arrow[hook]{r}{id_r} \arrow[hook]{d}{id_l} & \dfr \arrow[hook]{d}{inc'} \\
U(\dfn) \arrow[hook]{r}{inc} & U(\df)
\end{tikzcd}
\end{center}
% \begin{gather*}
%     inc \circ id_l = inc' \circ id_r
% \end{gather*}
% https://math.stackexchange.com/questions/3446450/relation-between-monomorphicity-and-faithfulness

Now consider any other semicategory $\cb'$ equipped with
% monic?
% injective on objects faithful?
semifunctors:
\begin{gather*}
    l: \cb' \rightarrow U(\dfn)
    \qquad
    r: \cb' \rightarrow \dfr
\end{gather*}
such that the following diagram commutes:
\begin{center}
\begin{tikzcd}[column sep=2in,row sep=2in]
\cb' \arrow{r}{r} \arrow{d}{l} & \dfr \arrow[hook]{d}{inc'} \\
U(\dfn) \arrow[hook]{r}{inc} & U(\df)
\end{tikzcd}
\end{center}
Since $inc$ and $inc'$ are inclusion maps, $l$ and $r$ must act identically on objects and morphisms. Therefore, any object or morphism in the image of $l$ must also be in $\dfr$ and any object or morphism in the image of $r$ must also be in $U(\dfn)$. Therefore, any object or morphism in the image of either $l$ or $r$ must also be in $\cbr$. 

We can therefore define a semifunctor $h: \cb' \rightarrow \cbr$ that has the same action on objects and morphisms as $l$ and $r$. This implies that
\begin{gather*}
    id_l \circ h = l
    \qquad
    id_r \circ h = r.
\end{gather*}
Since $h$ must have the same actions on objects and morphisms as $l$ and $r$ to satisfy these equations it must be unique, and therefore $(\cbr, id_l, id_r)$ is the pullback of the diagram:
\begin{gather*}
    U(\dfn) \xhookrightarrow{inc} U(\df) \xhookleftarrow{inc'} \dfr
\end{gather*}
Next, consider some $f: \Omega^n \times \rl^p \times \rl^a \rightarrow \rl^b$ in $\cbr$, and note that for any $x_p \in \rl^p, x_a \in \rl^a, j \leq b$, the random variable $f(\_,x_p, x_{a})[j]: \Omega^^n \rightarrow \rl$ is univariate normal (Proposition \ref{proposition:multivariate-normal-closure}). For each $n \in \mathbb{N}$ we also define the standard deviation function $s_n: (\Omega^n \rightarrow \rl) \rightarrow \rl$ where for $g: \Omega^n \rightarrow \rl$:
\begin{gather*}
    s_n(g) = \sqrt{E_{\mu_n}[(g - E_{\mu_n}[g])^2]}.
\end{gather*}
%
% NOTE: The reason why we need to restrict to C_R for this is that the following term is not defined when the variance is zero, because of division by zero issues
Now for any $x_p \in \rl^p, x_a \in \rl^a, y\in\rl, j\leq b$ we can write:
\begin{gather*}
    log \ \frac{d f(\_,x_p, x_{a})[j]_{*}\mu^n}{d\lambda}: \rl \rightarrow \rl
\end{gather*}
\begin{align*}
    log \left( \frac{d f(\_,x_p, x_{a})[j]_{*}\mu^n}{d\lambda}(y) \right)= \\
    %
    % definition of normal
    log \left(\left( \frac{1}{s_n(f(\_,x_p, x_a)[j])\sqrt{2\pi}} \right)
    \exp\left(-\frac{\left(y - E_{\mu^n}[f(\_,x_p, x_a)[j]]\right)^2}{2s_n(f(\_,x_p, x_a)[j])^2}\right) \right)= \\
    %
    % log normal
    - \frac{log(2\pi s_n(f(\_,x_p, x_a)[j])^2)}{2} - \frac{1}{2 s_n(f(\_,x_p, x_a)[j])^2} \left(y - E_{\mu^n}[f(\_,x_p, x_a)[j]]\right)^2.
\end{align*}
Therefore:
\begin{align*}
    & \alpha_n(g) = - \frac{log(2\pi s_n(g)^2)}{2}
    \\
    & \beta_n(g) = \frac{1}{2 s_n(g)^2} 
    \\
    & er(a,b) = (a-b)^2.
\end{align*}
\end{proof}

\subsection{Backpropagation}
The arrows in a Marginal Likelihood Factorization Category $\cb$ are equipped with the structure that we need to derive both an optimization objective and a learning procedure. Therefore, for any Marginal Likelihood Factorization Category $\cb$ and choice of learning rate $\epsilon$ we can define a backpropagation functor into Fong et al.'s \cite{fong2019backprop} $\learn$ category.
\begin{definition}
% NOTE: The image of \expectation is Para(EucMeas), which is not actually the same as Fong et al.'s Para category since the arrows in Fong et al.'s Para (FongPara) category are equivalence classes. In order to apply the Learn functor in earnest we first need to embed the morphisms in Para(EucMeas) into the equivalence classes in FongPara that contain them. This embedding process should be functorial since each morphism in Para(EucMeas) belongs to exactly one equivalence class and the composition of equivalence classes is defined in terms of representatives
% 
Write $F_{er}$ for Fong et al.'s \cite{fong2019backprop} Backpropagation functor with learning rate $\epsilon$ under the marginal error function $er$ of $\cb$. We define the functor $E_{er}$ to map a parametric statistical model in $\cb$ to a learning algorithm:
\begin{gather*}
    E_{er}: \cb \rightarrow \learn \\
    E_{er} = F_{er} \circ \expectation
\end{gather*}
\end{definition}
Where $\expectation$ is defined in Proposition \ref{proposition:expectation-functor}. For example, $E_{er}$ sends parametric statistical models in $\dfn$ to learning algorithms that minimize the square error function with gradient descent. We can think of $E_{er}$ as a point estimation functor: it sends an arrow $f$ in $\cb$ to a learner whose inference function is formed from $f$'s expectation. The higher order moments of the pushforward distributions of the arrows in $\cb$ are then used to define the loss function $er$.

\section{Closing Thoughts on Categorical Stochastic Processes and Likelihood}

Consider once again a physical system that is composed of several components, each of which has some degree of aleatoric uncertainty. If we construct a neural network model for this system directly, we cannot characterize the interactions between the uncertainty in the different parts of the system. However, if we model the components of the system as stochastic processes and apply $\df$ composition, we can capture how the uncertainty of the component parts combine. For example, given estimates of the kind of uncertainty inherent to the photorecepters in the eye, edge-detecting neurons in primary visual cortex, and higher-order feature detectors in the later stages of visual cortex, we may be able to build a more realistic model of how these sources of uncertainty interact than the one that Eberhardt et al. \cite{eberhardt2016deep} use to assess how the visual cortex performs a rapid stimulus categorization task.

Once we build such a model, we can use $E_{er}$ to derive a Learner with a structure that incorporates this combined uncertainty. This functor will convert the model to a point estimator and bundle the combined uncertainty into a loss function.
% In contrast, $P_{er}$ will preserve the uncertainty and produce a learning algorithm where both forward and backward passes are stochastic.

One of the largest differences between this construction and those of Cho et al. \cite{cho2019disintegration} and Culbertson et al. \cite{culbertson2013bayesian} is the treatment of model updates in the face of new data. While these authors also describe categorical frameworks in which we can model how a new observation updates the parameters of a statistical model, they primarily study Bayesian algorithms in which the model parameters are represented with a probability distribution.

In contrast, our construction is inherently frequentist. While the backpropagation functor above aims to find an optimal parameter value given the data we have seen, it makes no assumptions about what that value may be. Although uncertainty motivates the objective that our parameter estimation procedure aims to optimize, the optimization algorithm does not use it directly. Therefore, a potential future direction for this work is to extend the category $\df$ of deterministic and frequentist models to handle generative algorithms that model uncertainty in the input vector and Bayesian algorithms that model uncertainty in the parameter vector.

Furthermore, our current definition of Marginal Likelihood Factorization Categories may be overly restrictive.
% Ifinallt is unclear whether these categories are the most general class of subcategories of $\df$ for which the functor $E_{er}: \cb \rightarrow \learn$ is defined.
For example, our definition specifies that each category is characterized by a single marginal error function $er$. This makes it challenging to build a theory for how we could compose Marginal Likelihood Factorization Categories with different marginal error functions. Another potential future direction would be to relax the restrictions on these categories or prove that they are necessary.